\documentclass{article}

\usepackage{microtype}
\usepackage{graphicx, subcaption}
\usepackage{booktabs} % for professional tables

\usepackage{hyperref}

\usepackage[accepted]{icml2023}

\usepackage{amsmath}
\usepackage{amssymb}
\usepackage{mathtools}
\usepackage{amsthm}

\usepackage[capitalize,noabbrev]{cleveref}

\theoremstyle{plain}
\newtheorem{theorem}{Theorem}%[section]
\newtheorem{proposition}[theorem]{Proposition}%
\newtheorem{lemma}[theorem]{Lemma}
\newtheorem{corollary}{Corollary}[theorem]
\newtheorem{definition}{Definition}%[theorem]
\newtheorem{remark}{Remark}%[theorem]

\usepackage[textsize=tiny]{todonotes}

\newcommand{\new}[1]{#1}

\newcommand{\SA}[1]{{\color{red}}}
\newcommand{\LL}[1]{{\color{blue}}}
\newcommand{\ap}[1]{{\color{orange}}}
\newcommand{\ml}[1]{{\color{purple}}}

\newboolean{arxivFormat}
\setboolean{arxivFormat}{false}

\usepackage{amsmath,amsfonts,bm,mathtools}
\usepackage{ifthen}

\newtheorem{problem}{Problem} %[section]

\newcommand{\card}[1]{
\ifthenelse{\equal{#1}{}}{Empty.}{Nonempty.}
}

\newcommand{\z}[1]{z_{#1}}
\newcommand{\prez}[1]{\zeta_{#1}}
\newcommand{\weight}[1]{W_{#1}}

\newcommand{\act}[2]{\phi_{#1}(#2)}
\newcommand{\nn}[1]{f^{w}\left(#1\right)}

\newcommand{\bz}[1]{\bm{z}_{#1}}

\newcommand{\bprez}[1]{\bm{\zeta}_{#1}}
\newcommand{\bprezElem}[2]{\bm{\zeta}^{(#2)}_{#1}}
\newcommand{\bweight}[1]{\bm{W}_{#1}}

\newcommand{\meanWeight}[2]{
\ifthenelse{\equal{#1}{} \AND \equal{#2}{}}{m_w}
{\ifthenelse{\equal{#2}{}}{m_{#1}}
{m_{w,#1,#2}}}}
\newcommand{\varWeight}[2]{
\ifthenelse{\equal{#1}{} \AND \equal{#2}{}}{\Sigma_w}
{\ifthenelse{\equal{#2}{}}{\Sigma_{w,#1}}
{\Sigma_{w,#1, #2}}}}

\newcommand{\meanBias}[2]{
\ifthenelse{\equal{#1}{} \AND \equal{#2}{}}{\mu_b}
{\ifthenelse{\equal{#2}{}}{\mu_{b,#1}}
{\mu_{b,#1,#2}}}}

\newcommand{\varBias}[2]{
\ifthenelse{\equal{#1}{} \AND \equal{#2}{}}{\sigma_b}
{\ifthenelse{\equal{#2}{}}{\sigma_{b,#1}}
{\sigma_{b,#1,#2}}}}

\newcommand{\bnn}[1]{
\ifthenelse{\equal{#1}{}}{f^{\bm{w}}}{f^{\bm{w}}\left(#1\right)}}
\newcommand{\bnnUntil}[2]{f_{#1}^{\bm{w}}(#2)}

\newcommand{\matCoefL}[0]{\check{A}}
\newcommand{\matCoefU}[0]{\hat{A}}

\newcommand{\vecCoefL}[0]{\check{a}}

\newcommand{\vecCoefU}[0]{\hat{a}}

\newcommand{\vecBiasL}[0]{\check{b}}
\newcommand{\biasL}[0]{\check{\beta}}
\newcommand{\vecBiasU}[0]{\hat{b}}
\newcommand{\biasU}[0]{\hat{\beta}}

\newcommand{\E}[2]{\mathbb{E}_{#1}\left[#2\right]}
\newcommand{\Prob}[2]{\mathbb{P}_{#1}\left[#2\right]}
\newcommand{\condE}[3]{\mathbb{E}_{#1}\left[#2\mid#3\right]}
\newcommand{\nncondE}[3]{\Tilde{\mathbb{E}}_{#1}\left[#2\mid#3\right]}

\newcommand{\nDist}[2]{\mathcal{N}\left(#1;\ #2 \right)}
\newcommand{\nProbDens}[3]{\mathcal{N}\left(#1 \mid #2;\ #3 \right)}
\newcommand{\nCDF}[3]{\Phi\left(#1\mid#2;#3 \right)}
\newcommand{\supp}[1]{\text{supp}(#1)}

\newcommand{\rectnDist}[2]{\mathcal{N}_{\mathcal{R}} \left(#1;\ #2 \right)}
\newcommand{\rectnProbDens}[3]{\mathcal{N}_{\mathcal{R}}\left(#1 \mid #2;\ #3 \right)}

\newcommand{\mat}[2]{\left(\begin{array}{#1}#2\end{array}\right)}

\newcommand{\val}[2]{\ifthenelse{\equal{#2}{}}{V_{#1}}{V_{#1}(#2)}}
\newcommand{\valElem}[3]{\ifthenelse{\equal{#3}{}}{V^{(#1)}_{#1}}{V^{(#1)}_{#2}(#3)}}
\newcommand{\valU}[2]{\ifthenelse{\equal{#2}{}}{\hat{V}_{#1}}{\hat{V}_{#1}(#2)}}
\newcommand{\valElemU}[3]{\ifthenelse{\equal{#3}{}}{\hat{V}^{(#1)}_{#1}}{\hat{V}^{(#1)}_{#2}(#3)}}
\newcommand{\valL}[2]{\ifthenelse{\equal{#2}{}}{\check{V}_{#1}}{\check{V}_{#1}(#2)}}
\newcommand{\valElemL}[3]{\ifthenelse{\equal{#3}{}}{\check{V}^{(#1)}_{#1}}{\check{V}^{(#1)}_{#2}(#3)}}

\newcommand{\erf}[1]{\text{erf}\left(#1\right)}
\newcommand{\inverf}[1]{\text{erf}^{-1}\left(#1\right)}
\newcommand{\pospart}[1]{[#1]_{+}}
\newcommand{\negpart}[1]{[#1]_{-}}
\newcommand{\relu}[1]{\text{ReLU}\left(#1\right)}
\newcommand{\softmax}[0]{\text{softmax}}
\DeclareMathOperator*{\argmax}{arg\,max}

\newcommand{\dataset}{{\cal D}}

\newcommand{\Exp}[1]{\exp \left(#1\right)}

\newcommand{\diag}[1]{\text{diag}\left(#1\right)}

\icmltitlerunning{BNN-DP: Robustness Certification of Bayesian Neural Networks via Dynamic Programming}

\begin{document}

\twocolumn[
\icmltitle{BNN-DP: Robustness Certification of Bayesian Neural Networks via Dynamic Programming}



\icmlsetsymbol{equal}{*}

\begin{icmlauthorlist}
\icmlauthor{Steven Adams}{delft}
\icmlauthor{Andrea Patanè}{dublin}
\icmlauthor{Morteza Lahijanian}{colorado}
\icmlauthor{Luca Laurenti}{delft}
% \icmlauthor{Firstname2 Lastname2}{equal,yyy,comp}
\end{icmlauthorlist}

\icmlaffiliation{delft}{Delft Center for Systems and Control, Technical University of Delft, Delft, 2628 CD, The Netherlands}
\icmlaffiliation{dublin}{School of Computer Science and Statistics, Trinity College Dublin, Dublin 2, Ireland}
\icmlaffiliation{colorado}{Departement of Aerospace Engineering Sciences and Computer Science, University of Colorado Boulder, Boulder, CO 80303, USA}

\icmlcorrespondingauthor{Steven Adams}{s.j.l.adams@tudelft.nl}

\icmlkeywords{Bayesian Neural Networks, Adversarial Robustness, Certification}

\vskip 0.3in
]



\printAffiliationsAndNotice{}  % leave blank if no need to mention equal contribution

\begin{abstract}
In this paper, we introduce BNN-DP, an efficient algorithmic framework for analysis of adversarial robustness of Bayesian Neural Networks (BNNs).
Given a compact set of input points $T\subset \mathbb{R}^n$, 
BNN-DP computes lower and upper bounds on the BNN's predictions for all the points in $T$. 
The framework is based on an interpretation of BNNs as stochastic dynamical systems, which enables the use of Dynamic Programming (DP) algorithms to bound the prediction range along the layers of the network.  Specifically, the method uses bound propagation techniques and convex relaxations to derive a backward recursion procedure to over-approximate the prediction range of the BNN with piecewise affine functions. 
The algorithm is general and can handle both regression and classification tasks.
On a set of experiments on various regression and classification tasks and BNN architectures, we show that BNN-DP outperforms state-of-the-art methods by up to four orders of magnitude in both tightness of the bounds and computational efficiency.
\end{abstract}

\section{Introduction}

Adversarial attacks (small and often imperceptible perturbations to input points that can trigger incorrect decisions) have raised serious concerns about the robustness of models learned from data \citep{biggio2018wild,goodfellow2014explaining}. %The development of methods for quantifying the adversarial robustness of machine learning models, is crucial for their use in safety-critical scenarios, as test accuracy fails to account for the behavior of a model in adversarial settings.  
Bayesian Neural Networks (BNNs), i.e., neural networks with distributions placed over their parameters, have been proposed as a potentially more robust machine learning paradigm compared to their deterministic 
% (i.e., trained with SGD\ap{Not a big fan of this paranthesis. SGD is an acronym not explained, but anyway it is not SGD that makes the NN deterministic, you can train by other means and still get deterministic stuff. At best we could change it with `e.g.', but we could, imo, also remove altogether}) 
counterpart \citep{carbone2020robustness, mcallister2017concrete}.
While retaining the advantages intrinsic to deep learning (e.g., representation learning), BNNs enable principled evaluation of model uncertainty, which can be used for flagging out-of-distribution samples and robust decision making \cite{kahn17}. 
However, existing methods that formally (i.e., with certified bounds) evaluate the robustness of BNNs \cite{berrada2021make, wicker2020probabilistic, lechner2021infinite} are 
% either 
limited to posterior distributions with bounded support, thus not supporting the majority of the algorithms commonly employed to train BNNs \citep{blundell2015weight,zhang2018noisy,osawa2019practical}
and 
% or 
lack scalability to BNNs with non-negligible posterior variance estimates. 
% \SA{They all consider bounded support and hence lack scalability right? \citep{wicker2020probabilistic} can not sample the full unbounded support}
% \ap{Cite something here}

In this paper, we present BNN-DP, a novel algorithmic framework that quantifies the adversarial robustness of BNNs with formal (strong) guarantees.  BNN-DP is scalable and supports posterior distributions of unbounded support, as commonly used in BNNs, e.g., Gaussian distributions. We consider both regression and classification settings. For a compact set of input points $T\subset \mathbb{R}^{n_0}$, we study the robustness of the BNN's decision, i.e., argmax of the expectation of the softmax in case of classification and expectation of the output of the BNN for regression, for all the points in $T.$ 
As exact computation of these quantities is infeasible \citep{berrada2021make}, we focus on computing piecewise affine (PWA) upper and lower bounds.
Inspired by \citet{marchi2021training}, we take a unique view of BNNs as stochastic dynamical systems that evolve over the layers of the neural network and show that the computation of the BNN robustness can be formulated as the solution of a Dynamic Program (DP). This allows us to break the computation of adversarial robustness into a set of simpler sub-problems (one for each layer of the BNN).  
%containing $x^*$ we consider the decision of a BNN and prove that th 
%perturbations in a  both regression and classification settings 
%as the invariance of the decision in a small neighbourhood of a test point \citep{huang2017safety}, and thus study the worst-case effect of bounded perturbations of the input on the BNN optimal decision. \LL{Small comment, let's not cite Ruan as the paper has errors in the maths.} Similarly to \citep{ruan2018reachability}, we observe that, to provide provable guarantees on the model prediction over $T$, it suffices to compute the minimum and maximum of the reachable prediction range. 
%Unfortunately, exact direct computation requires solving a non-convex optimization problem, which is generally infeasible \citep{neumaier2004complete}. 
% of the extremes of the prediction range can be framed as a stochastic dynamic programming problem, splitting the problem into smaller (still non-convex) sub-problems. 
Critically, while each of these problems is still possibly non-convex, we show that accurate and efficient PWA relaxations can be derived for each  by relying on tools from Gaussian processes and convex optimizations.

%linear bound propagation techniques analogous to those commonly employed to certify non Bayesian NNs to compute convex relaxations of the sub-problems w.r.t. the unbounded support of the learned model in the function space of the BNNs. 
% \SA{split in two steps, but need to be careful not to raise suspiciousness about approach by explaining gaussian properties used. Introduce IBP and LBP}
%The method we propose is anytime, and hence provides worst case bounds on the adversarial robustness. Our framework can handle robustness for both regression and classification tasks.

% Under idealist conditions BNNs provide a natural protection against adversarial attacks, small, often imperceptible perturbations to their inputs that can trigger corrupt decision making, however it has been shown that in practice BNNs can be easily fooled \citep{grosse2018limitations,athalye2018obfuscated}. 

% While retaining the advantages of standard (deterministic) Neural Netorks (NNs), Bayesian neural networks (BNNs), i.e., neural networks with distributions placed over their weights and biases, enable principled quantification of their predictions' uncertainty \citep{neal2012bayesian}. 
% Intuitively, the latter can be used to provide a natural protection against adversarial attacks; small, often imperceptible perturbations to their inputs that break performance, making them particularly appealing for safety-critical scenarios, in which the safety of the system must be provable guaranteed \citep{mcallister2017concrete}. 

We validate  our framework on several regression and classification tasks, including the Kin8nm, MNIST,  Fashion MNIST, \new{and CIFAR-10} datasets, and a range of BNN architectures. 
For all tasks, the results show that our method outperforms state-of-the-art competitive approaches in both precision and computational time.
%are able to provide significantly stronger guarantees when comparing to prior work, while being more computationally efficient - \SA{would you include specific examples of results?} for instance, for \SA{example}. Similarly, for a BNN trained on MNIST we improve the guaranteed $\ell_{\infty}$ robustness from $.. \rightarrow ..$
For instance, on the Fashion MNIST dataset, our approach achieves an average 93\% improvement in certified lower bound compared to \citet{berrada2021make}, while being around  3 orders of magnitude faster. 
In summary, this paper makes the following main contributions:
\begin{itemize}
    \item we introduce a framework based on stochastic dynamic programming and convex relaxation for the analysis of adversarial robustness of BNNs, 
    \item we implement an efficient algorithmic procedure of our framework for BNNs trained with Gaussian variational inference (VI) in both regression and classification settings,\footnote{
    \new{Our code is available at \url{https://github.com/sjladams/BNN_DP}.}
    \label{footnote:code}}
     and 
    % \LL{Need to think at a second bullet point. We may want to emphasize our results for Gaussian VI in here. In this way we also make clear that some of our results are specific for that setting}. 
    % \ml{say something about extending the state of the art to unbounded support for posterior distribution via Gaussian VI}
    \item we benchmark the robustness of a variety of BNN models on five datasets, empirically demonstrating how our method outperforms state-of-the art approaches by orders of magnitude in both tightness and efficiency. 
\end{itemize}

% \LL{Stress either in abstract or intro that we are able to derive for BNNs linear bound propagation techniques analogous to those commonly employed to certify for non-Bayesian NNs}

% \LL{THis we do not need to discuss it in the related works, but we can mention it in the intro togheter with other methods for robust training of BNNs }
% In \citep{wicker2021bayesian} a principled Bayesian approach was proposed for incorporating adversarial robustness in the posterior inference procedure for BNNs, whereas we focus on providing guarantees, irrespective of the BNNs being trained to make them more easily verifiable. \SA{Incorporate remark that our be applied for such purpose.}

\paragraph{Related Works}
Many algorithms have been developed for certification of deterministic (i.e., non Bayesian) neural networks (NNs) \citep{katz2017reluplex,weng2018towards,wong2018provable, bunel2020lagrangian}. 
% In particular, \citep{weng2018towards, zhang2018efficient, gowal2018effectiveness}  apply reachability analyses using interval and symbolic propagation techniques, and \citep{dvijotham2018dual, wong2018provable, bunel2020lagrangian} use tools from dual optimization. 
However, these methods cannot be employed to BNNs because they all assume the weights of the network have a fixed value, whereas in the Bayesian setting they are distributed according to the BNN posterior.  Methods for certification of BNNs have recently presented in \citep{wicker2020probabilistic,berrada2021make,lechner2021infinite}. \citet{wicker2020probabilistic} consider a different notion of robustness than the one in this paper, not directly related to adversarial attacks on the BNN decision. Furthermore, that work considers a partitioning procedure in weight space that makes it applicable only to small networks and/or with small variance. 
The method proposed in \citep{berrada2021make} is based on dual optimization.  Hence, it is restricted to distributions with bounded support and needs to solve non-convex problems at large computational costs for classification tasks.
Separately, \citet{lechner2021infinite} aims to build an intrinsic safe BNN by truncating the posterior in the weight space.  
\citet{cardelli2019statistical,wicker2021bayesian} introduced statistical approaches to quantify the robustness of a BNN, which however, does not return formal guarantees, which are necessary  in safety-critical settings. 
Empirical methods that use the uncertainty of BNNs to flag adversarial examples are introduced in \citep{rawat2017adversarial, smith2018understanding}. These, however, consider only point-wise uncertainty estimates, specific to a particular test point and do not account for worst-case adversarial perturbations. 

% \LL{IF we need space, the following paragraph can be removed}

\new{Various recent works have proposed formal methods to compute adversarial robustness for Gaussian Processes (GPs) \citep{cardelli2018robustness, smith2018understanding, patane2022adversarial,smith2022adversarial}. In BNNs, however, due to the non-linearity of activation functions, the distribution over the space of functions induced by a BNN is generally non-Gaussian, even if a Gaussian distribution in weight space is assumed. Hence, the techniques that are developed for GPs cannot be directly applied to BNNs.}

\section{Robust Certification of BNNs Problem} \label{sec:ProbForm}
\subsection{Bayesian Neural Networks (BNNs)}
\label{sec:BNNs}
For an input vector $x\in \mathbb{R}^{n_0}$, we consider fully connected neural networks $f^w:\mathbb{R}^{n_0} \to \mathbb{R}^{n_{K+1}}$
of the following form for $k=0,\ldots, K$:\footnote{\new{Note that the formulation of neural networks considered in Eqn~\eqref{eq: nn} also includes convolutional neural networks (CNNs). In fact, the convolutional operation can be interpreted as a linear transformation into a larger space; see, e.g., Chapter 3.4.1 in \citep{gal2016dropout}. This allows us to represent convolutional layers equivalently as fully connected layers, and do verification for CNNs as we  show in Section \ref{sec:ExperimentalResults}.}}
\begin{equation}\label{eq: nn} 
\begin{aligned}
    &\z{0} =  x, &&
    \prez{k+1} = \weight{k}(\z{k}^T, 1)^T, \\ 
    &\z{k} = \act{k}{\prez{k}}, && \nn{x}=\prez{K+1},  
\end{aligned}
\end{equation}
 where $K$ is the number of hidden layers, $n_k$ is the number of neurons of layer $k$, $\phi_k:\mathbb{R}^{n_k}\to \mathbb{R}^{n_k}$ is a vector of continuous activation functions %\footnote{Our method extends to piecewise continuous activation functions with finitely many pieces.\label{footnote:actFuncs}}
 (one for each neuron) in layer $k$, and $\weight{k} \in \mathbb{R}^{n_{k} \times n_{k+1}}$ is the matrix of weights and biases that correspond to the $k$th layer of the network.
We denote the vector of parameters by $w = (\weight{0}^T,\hdots,\weight{K}^T)^T$ and 
the mapping from $\prez{k_1}$ to $\prez{k_2}$
by $f_{k_1:k_2}^w:\mathbb{R}^{n_{k_1}}\rightarrow \mathbb{R}^{n_{k_2}}$ for $k_1,k_2\in \{0,...,K\}$.
% , i.e., $f^w_{k_1:k_2}(\prez{k_1})=\weight{k}(\act{k}{\prez{k}}^T, 1)^T$.
 $\prez{K+1}$ is the final output of the network (or the logit in the case of classification problems). 

In the Bayesian setting, one starts by assuming a prior distribution $p(w)$ over the parameters $w$ and a likelihood function $p(y \vert x,w)$. We adopt bold notation to denote random variables and write $\bnn{}$ to denote a BNN defined according to Eqns.~\eqref{eq: nn}. The likelihood is generally assumed to be Gaussian in case of regression and 
categorical for classification, where the probability for each class is given as the softmax of the neural network final logits  \citep{mackay1992practical}.
Then, given a training dataset $ \dataset =\{(x_i,y_i)\}_{i=1}^{N_\mathcal{D}}$, learning amounts to computing the posterior distribution $p (w \vert \dataset )$ 
via the Bayes rule \citep{mackay1992practical}. 
The posterior predictive distribution over an input $x^*$ is finally obtained by marginalising the posterior over the likelihood, i.e., $p(y^*|x^*,\mathcal{D}) = \int p(y^*|x^*,w)p(w|\mathcal{D})dw.$ 
%\begin{align}\label{eq:posterior_predictive}
%    p(y^*|x^*,\mathcal{D}) = \int p(y^*|x^*,w)p(w|\mathcal{D})dw.
%\end{align}
% \ml{what's $y^*$?}
%Notice that the posterior predictive distribution is still a distribution over the output space.
The final output (decision) of the BNN, $\hat{y}(x^*)$, is then computed by minimising 
% the expected value of the loss function; that is:
a loss function $\mathcal{L}$ averaged over the predictive distribution, i.e., 
% \ap{$\mathcal{L}$ is a function, you can't marginalise it out. I think you mean marginalise the predictive posterior over the loss function, but I'm not sure I'd use this expression either, since you have here only one source of uncertainty, that coming from $y^*$ – so maybe just averaging is a better term}\LL{If we want to be super precise I would say averaged. Marginalised could also work in this context, but could indeed be a stretch (you generally talk about marginalisation when you have a joint distribution and compute the marginals). }
% \ml{Can we say the expected value of the loss function?}
$$\hat{y}(x^*) = \arg\min_{y} \int \mathcal{L}({y,y^*})p(y^*|x^*,\dataset)dy^*.$$
In this paper, we focus on both regression and classification problems. In regression, an $l_2$ loss is generally used, which leads to an optimal decision $\hat{y}(x^*)$  given by the mean of the predictive posterior distribution \citep{neal2012bayesian}, 
i.e., $\hat{y}(x^*) = \E{\mathbf{y} \sim p(y|x^*,\dataset)}{\mathbf{y}}.$\footnote{In the remainder, we may omit the probability measure of an expectation or probability when it is clear from the context.} 
% \ml{to be consistent with claasification below, use $\E{\bm{w} \sim p(w|\dataset)}{\bnn{x^*}}$?}
% \LL{I would also add a remark or footnote saying that our approach can be extended to other losses.}
% We obtain that: \ap{How? You mean by employing the $\ell_{0-1}$ loss} \LL{It is my fault in here as I have asked Steven to siplify and remove as many not strictly needed background information as possible, but yes, if it makes things easier to understand, let's indeed add using the $l_{0-1}$ loss}
% \ml{suggestion: For classification, $\ell_{0-1}$ loss is typically employed, which results in the output}
For classification, $\ell_{0-1}$ loss is typically employed, which results in
$$\hat{y}(x^*) = \argmax_{i \in \{1,\hdots,n_{K+1}\}} \E{\bm{w} \sim p(w|\dataset)}{\softmax^{(i)}(\bnn{x^*})},$$ where 
$\softmax^{(i)}$ is the $i$th component of the $n_{K+1}$-dimensional 
% \ap{check what n should be} 
softmax function.\footnote{Analogous formulas can be obtained for the weighted classification loss by factoring in misclassification weights in the argmax.}
Unfortunately, because of the non-linearity introduced by the neural network architecture, the computation of the posterior distribution and consequently of
$\hat{y}(x^*)$ cannot be done analytically. Therefore, approximate inference methods are required. In what follows, we focus on mean-field Gaussian Variational Inference (VI) approximations \citep{blundell2015weight}.  Specifically, we fit an approximating multivariate Gaussian
$q(w) = \nProbDens{w}{\mu_w}{\Sigma_w}\approx p(w\mid \dataset) $
%\begin{align}
%    q(w) &= \nProbDens{w}{\mu_w}{\Sigma_w}\approx p(w\mid \dataset) \label{eq: mean-field distribution}
%\end{align}
with mean $\mu_w$ and block diagonal covariance matrix $\Sigma_w$ such that for $k\in\{0,\hdots,K\}$ and $i\in\{1,\hdots,n_k\}$, the approximating distribution of the parameters corresponding to the $i$th node of the $k$th layer is
\begin{equation}\label{eq:qForNode}
q(\weight{k}^{(i,:)})=\nDist{\new{\weight{k}^{(i,:)}} \mid \mu_{w,k,i}}{\Sigma_{w,k,i}}  
\end{equation} 
with mean $\mu_{w,k,i}$ and covariance matrix $\Sigma_{w,k,i}$. \footnote{\new{BNN-DP can be extended to Gaussian approximation distributions with inter-node or inter-layer correlations. In that case, to solve the backward iteration scheme of Theorem \ref{Theorem:ValueIteration}, the value functions need to be marginalized over partitions in weight space.}}

% \begin{remark}
%     \new{
%     Even though we focus on mean-field Gaussian approximations because of their widespread use and existing closed-form solutions to the integrals of Gaussian distributions, the methods introduced in this paper can be extended to other forms of $q(w)$.
%     The dynamic programming problem in Theorem \ref{Theorem:ValueIteration} is then solved via Monte Carlo sampling techniques.
%     % as long as the block diagonal structure of the covariance matrix of the approximated posterior distribution is preserved \citep{betancourt2015hamiltonian}. 
%     }
% \end{remark}
\begin{remark}
    % Even though, we focus on Gaussian VI because of its widespread use and existing closed-form solutions to the integrals of Gaussian distributions, the methods introduced in this paper can be extended to other forms of $q(w)$ and different approximate inference methods, such as HMC \citep{neal2012bayesian} or Dropout
    % \citep{gal2016dropout}. 
    % In these cases, the dynamic programming problem in Theorem \ref{Theorem:ValueIteration} need to be solved via Monte Carlo sampling techniques.
    \new{While our primary focus is on VI, the techniques presented in this paper can be applied to other approximate inference methods, such as HMC \citep{neal2012bayesian} and Dropout \citep{gal2016dropout}.
    In these cases, 
    % the approximate posterior distribution is discrete rather than continuous; hence 
    the prediction of the BNN is obtained by averaging over a finite ensemble of NNs. 
    For this setting, the dynamic programming problem in Theorem \ref{Theorem:ValueIteration} reduces to computing piecewise linear relaxations for each layer of a weighted sum, i.e., an average, of deterministic neural networks, and propagating the resulting relaxations backward.}
\end{remark}

\subsection{Problem Statement}
Given a BNN $\bnn{}$ trained on a dataset $\dataset$,
% \ml{I'm not sure what the common practice is in the community, but I'm getting a little confused by the notations... If I understand correctly, $\nn{}$ is an NN with weights $w$ and $\bnn{}$ is a BNN with posterior $p(w|D)$.  If so,
% I think we should clarify it, i.e., $y=\nn{x}$ and $\hat{y}(x) = E[\bnn{x}]$ for regression, and for classification, it's given by Eqn???} our goal is to study the adversarial robustness of $\bnn{}$. 
as common in the literature \citep{madry2017towards}, for a generic test point $x^*$, we represent the possible adversarial perturbations by defining a compact neighbourhood $T$ around $x^*$
%i.e., such that $x^* \in T$, 
and measure the changes in the BNN output caused by limiting the perturbations to lie within $T$. 
% \LL{Remove this sentence from here (see comment below)}In particular, we consider the following definition of adversarial robustness analogous to the standard notion of adversarial robustness employed for deterministic neural networks \citep{berrada2021make} and Gaussian processes \citep{patane2022adversarial}.
\begin{definition}[Adversarial Robustness]
\label{def:AdversarialRObustness}
Consider a BNN $\bnn{}$, a compact set $T \subset \mathbb{R}^{n_0}$, and input point $x^*\in T$.
% Let $T \subseteq \mathbb{R}^{n_0}$ be a compact set and $x^*\in T$, and consider a BNN $\bnn{}$. 
For a given threshold $\gamma > 0$, $\bnn{}$
is \textit{adversarially robust} in $x^*$ iff
\begin{equation*}
   \forall x\in T, \quad \| \hat{y}(x) - \hat{y}(x^*) \|_p\leq \gamma,  
\end{equation*}
where $\|\cdot \|_p$ is an $\ell_p$ norm. 
% \ap{Why the two variables? I guess since we compute max and min in T, we can tackle this notion. But when people say adversarial example they usually refer to when one of the two variable is fixed. Not saying that we need to use necessarily the standard definition, I'm just wondering if we get anything more from this. For classification I think they are exactly equivalent, but for regression it changes slightly the role of $\gamma$ \SA{For simplification of presentation I changed it tho the standard notion.}}
\end{definition}

\noindent
% \ml{The first part of this sentence just repeats the definition above... we can remove it to be succinct, no?} 
Definition \ref{def:AdversarialRObustness} 
% defines adversarial robustness as the robustness of the output of a BNN for all $x\in T$ and 
is analogous to the standard notion of adversarial robustness employed for deterministic neural networks \citep{katz2017reluplex} and Bayesian models \citep{patane2022adversarial}. As discussed in Section \ref{sec:BNNs}, the particular form of a BNN's output depends on the specific application considered. Below, we focus on regression and classification problems. 

%Note that the Definition \ref{def:AdversarialRObustness} is given wrt to a function $h$, which associates a decision to the output of a neural network and is application dependent. For instance, in the case of regression problems with a symmetric loss (e.g., the square loss) $h$ is simply the identity function, while for classification problems with $0-1$ loss, $h$ would be the softmax function. 
%Furthermore, for a set $S\subseteq \mathbb{R}^m$ another case of particular interest is the case $h(x)=\mathbf{1}_{S}(x),$ where $\bm{1}_{S}$ is the indicator function for set $S$. In fact, in this case we obtain that $\E{\bm{w}\sim q(w)}{\bm{1}_{S}(\bnn{x})}=P_{\bm{w}\sim q(w)}(\bnn{x}\in S),$ that is the probability that $\bnn{x}$ lies in $X$. 

%In order to check if a BNN is adversarially robust in $x^*$, in the following problem we focus on finding the ranges of $E_{w\sim q(w)}[h_i(f^w(x))]$ for all $x\in T$.
%\SA{Split problem into one problem of regression for $\ell_2$-loss and one for classification with $0-1$ loss.}
\begin{problem}\label{prob:MainProblem}
Let $T \subset \mathbb{R}^{n_0}$  be a compact subset.  Define functions $I(y)=y$ and $\softmax(y)=[ \softmax^{(1)}(y),...,\softmax^{(n_{K+1})}(y) ]$
. Then, for a BNN $\bnn{}$, $h\in\{I,\softmax \}$, and $i\in \{1,...,n_{K+1}\}$, compute:
\begin{equation}\label{eqn:ProbEquation}
\begin{aligned}
    \pi_{\min}^{(i)}(T) = \min_{x \in T} \E{\bm{w}\sim q(\cdot)}{h^{(i)}(\bnn{x})}, \\
    \pi_{\max}^{(i)}(T) = \max_{x \in T} \E{\bm{w}\sim q(\cdot)}{h^{(i)}(\bnn{x})}.
\end{aligned}
\end{equation}
\end{problem}
In the regression case ($h=I$), Problem \ref{prob:MainProblem}
% \ap{Latex tip: when putting a period for shortening a word, latex will add an extra spacing thinking that it is a full stop. In order to escape that extra spacing, either put a backslash to escape it, or use a tilde between the period and the next word.} 
seeks to compute the ranges of the expectation of the BNN for all $x\in T$. Similarly, in the classification case ($h=\softmax$), Eqns.~\eqref{eqn:ProbEquation} define the ranges of the expectation of the softmax of each class for $x\in T$. 
It is straightforward to see that these quantities are sufficient to check whether $\bnn{}$ is adversarially robust for $x\in T$; that is, if $\sup_{x\in T}||\hat{y}(x)-\hat{y}(x^*)||_p\leq \gamma$. % (see \citep{ruan2018reachability}). 
% \ap{This is only a statement in this way. I'd (1) either phrased it as in "it is streightforward to see that blash blah" and put a citation of somebody that does it; or (2) you explain in a few words why}

% \LL{This remark is missing a part}
\begin{remark}
% \LL{This remark is a bit small. You could merge this with the remark about different losses, which would lead to different $h$.}
Our method can be extended to other losses, i.e., other forms of $h$ in Eqns.~\eqref{eqn:ProbEquation}, as long as affine relaxations of $h$ can be computed.
\end{remark}
\paragraph{Approach Outline}
Due to the non-convex nature of $\bnn{}$ and possibly $h$, the computation of $\E{\bm{w}\sim q(\cdot)}{h(\bnn{x})}$ is analytically infeasible. 
To solve this problem, in Section~\ref{sec:SDP}, we view BNNs as stochastic dynamical systems evolving over the layers of the neural network. Through this, we show that adversarial robustness can be characterized as the solution of a dynamic programming (DP) problem. 
This allows us to break its computation into $K$ simpler optimization problems, one for each layer. 
Each problem essentially queries a back-propagation of the uncertainty of the BNN through $h$ and from one layer of the neural network to the next. Due to the non-convex nature of the layers of the BNN, these problems still cannot be solved exactly. 
We overcome this problem by using convex relaxations. Specifically,
in Section \ref{sec:SolveDynamicProgram}, we show that efficient PWA relaxations can be obtained by recursively bounding the DP problem. In Section \ref{sec:Algo}, we combine the theoretical results into a general algorithm called BNN-DP that solves Problem \ref{prob:MainProblem} efficiently.
% via piecewise affine relaxations of a BNN.% The proofs of the theoretical results can be found in the Appendix. 
% \SA{All proofs can be founds in the Appendix}

\section{Preliminaries on Relaxations of Functions}
% \section{Preliminaries}
To propagate the uncertainty of the BNN from one layer to the other, we rely on upper and lower approximations of the corresponding Neural Network (NN), also known as \textit{relaxations}. 
For vectors $\check{x}, \hat{x}\in \mathbb{R}^n$, we denote by $[\check{x},\hat{x}]$ the $n$-dimensional hyper-rectangle defined by $\check{x}$ and $\hat{x}$, i.e., $[\check{x},\hat{x}] = [\check{x}^{(1)},\hat{x}^{(1)}] \times [\check{x}^{(2)},\hat{x}^{(2)}] \times \hdots \times [\check{x}^{(n)},\hat{x}^{(n)}]. $ 
% \ap{This part below in this section is a bit scholastic, I feel it can be made much sharper and shorter by going directly to the point, without the sotry of simple more refined or whatnot. That's an opinion though, see how you guys feel}
% \ml{I agree.  I'd substitute the sentence below with: ``We consider two types of relaxations, interval and affine, as defined below.'' I'd also delete the sentence between the two Defs.}
We consider two types of relaxations, interval and affine. %, as defined below.
% The simplest type of relaxation of a function is interval bounds where given an input set, the bounds are hyper-rectangle that contains all the outputs of the function for the points in the input set. 
\begin{definition}[Interval Relaxation]
An interval relaxation of a function $f:\mathbb{R}^n\rightarrow \mathbb{R}^m$ over a set $T\subseteq \mathbb{R}^n$
are two vectors $\vecBiasL,\vecBiasU\in\mathbb{R}^m$ 
such that
$f(x)\in [\vecBiasL, \vecBiasU]$ for all $x\in T$.
\end{definition}
% Another type of relaxation that often produces tighter bounds is affine relaxation, resulting in piece-wise affine bounds on the outputs of the function.
\begin{definition}[Affine Relaxation]
An affine relaxation of a function $f:\mathbb{R}^n\rightarrow\mathbb{R}^m$ over a set $T\subseteq \mathbb{R}^n$ are two affine functions $\matCoefL x+ \vecBiasL$ and $\matCoefU x+\vecBiasU$ with $\matCoefL, \matCoefU\in\mathbb{R}^{m\times n}$ 
and $\vecBiasL, \vecBiasU \in\mathbb{R}^m$ such that 
$f(x) \in [\matCoefL x+\vecBiasL, \matCoefU x+\vecBiasU]$ for all $x \in T$.
\end{definition}
Interval and symbolic arithmetic can be used to propagate relaxations through the layers of a NN.
Let, $\pospart{\alpha}\coloneqq \max\{\alpha,0\}$ and $\negpart{\alpha}\coloneqq \min\{\alpha,0\}$ 
represent the saturation operators on $\alpha$.
% the positive and negative parts of a real number $\alpha$
% \ml{what does a positive/negative part of a real number mean? Also, it sounds like $\pospart{\cdot}$ and $\negpart{\cdot}$ are operators, no?}
For a vector or matrix, $\pospart{\cdot}$ and $\negpart{\cdot}$ represent element-wise max and min, respectively. We adopt the notation of \citet{liu2021algorithms} and write interval arithmetic w.r.t.\ a linear mapping $M$ compactly as $\otimes$ where
$M \otimes [\vecBiasL,\vecBiasU] \coloneqq \big[ \pospart{M}\vecBiasL+\negpart{M}\vecBiasU, \ \pospart{M}\vecBiasU+\negpart{M}\vecBiasL \big],$
and use the similar notation for symbolic arithmetic.
% Further,
% for affine relaxation $\check{f}, \hat{f}:\mathbb{R}^n\rightarrow\mathbb{R}^m$ of function $f:\mathbb{R}^n\rightarrow\mathbb{R}^m$
% and an affine function $g:\mathbb{R}^m\rightarrow\mathbb{R}^l$, $g(z)=Az+b$, where $A\in\mathbb{R}^{m\times l}, b\in\mathbb{R}^l$, we denote the affine relaxation of $g(f(x))$ by
% $$g \otimes [\check{f}, \hat{f}] \coloneqq A \otimes [\check{f}(x),\hat{f}(x)] + b.$$
% % $$g(f(x)) \in A \otimes [\check{f}(x),\hat{f}(x)] + b, \ \forall x\in X.$$

% \ml{do we need this last notation here? it seems too verbose and kinda straightforward... why not define it when it's used.}

\section{BNN Certification via Dynamic Program}\label{sec:SDP}
% A Stochastic Dynamic Programming Approach to Certify BNNs
% \begin{figure}[t!]
%     \centering
%     % \vspace{-1mm}
%     \includegraphics[width=.8\columnwidth]{figures/theory/method.png}
%     % \vspace{-1mm}
%     \caption{
%     % \ap{In the image I believe that the last $z_3$ should be a pre-activation, ie $\zeta_3$ – really nice image btw}
%     Illustration of the DP algorithm in Theorem \ref{Theorem:ValueIteration} for a BNN with two hidden layers. Value functions $V_k$ are mappings from the latent and input spaces of the BNN to the mean of the output distribution. For each mapping, the distribution and mapping of a single point is displayed in orange. Starting from the last hidden layer, we recursively compute PWA approximations of the mappings. 
%     The true mean of the BNN for all $z_2\in\mathcal{Z}_2$ is in the green oval, which we over-approximate by the blue hexagon. 
%     }
%     \label{fig:SDP}
%     % \vspace{-2mm}
% \end{figure}

% \ml{suggestion for the section title: Dynamic Programming for BNN Certification... OR... BNN Certification via Dynamic Programming}

As observed in \citet{marchi2021training}, NNs and consequently BNNs can be viewed as dynamical systems evolving over the layers of the network. In particular, for $k\in \{0,...,K\}$, 
Eqn.~\eqref{eq: nn} can be rewritten as:
\begin{align} 
\label{Eqn:BNNsAsStochasticProcesses}
\z{k+1}=\act{k+1}{\bweight{k}(\z{k}^T, 1)^T} 
\end{align}
with initial condition $\z{0}=x$. 
Since, in a BNN, weights and biases are random variables sampled from the approximate posterior $q(\cdot)$, Eqn.~\eqref{Eqn:BNNsAsStochasticProcesses} describes a non-linear stochastic process evolving over the layers of the NN. 
% Following this observation, in Theorem \ref{Theorem:ValueIteration} we can write $\E{\bm{w}\sim q(\cdot)}{h(f^w(x))}$ as the solution of a stochastic dynamic programming (SDP) problem\LL{Maybe, to be more precise I would call it backward recursion problem. But we can check the literature in Bertsekas book and use the more appropriate name. }.%, as shown in Theorem \ref{Theorem:ValueIteration}.
% \ml{suggestion: The following theorem uses this observation to shows that $\E{\bm{w}\sim q(\cdot)}{h(f^w(x))}$ is the solution of a SDP problem, which can be solved using backward recursion.}
This observation leads to the following theorem, which shows that $\E{\bm{w}\sim q(\cdot)}{h(\bnn{x})}$ can be characterized as the solution to a backward recursion DP problem.

\begin{figure}[t!]
    \centering
    \vspace{2mm}
    \includegraphics[width=.8\columnwidth]{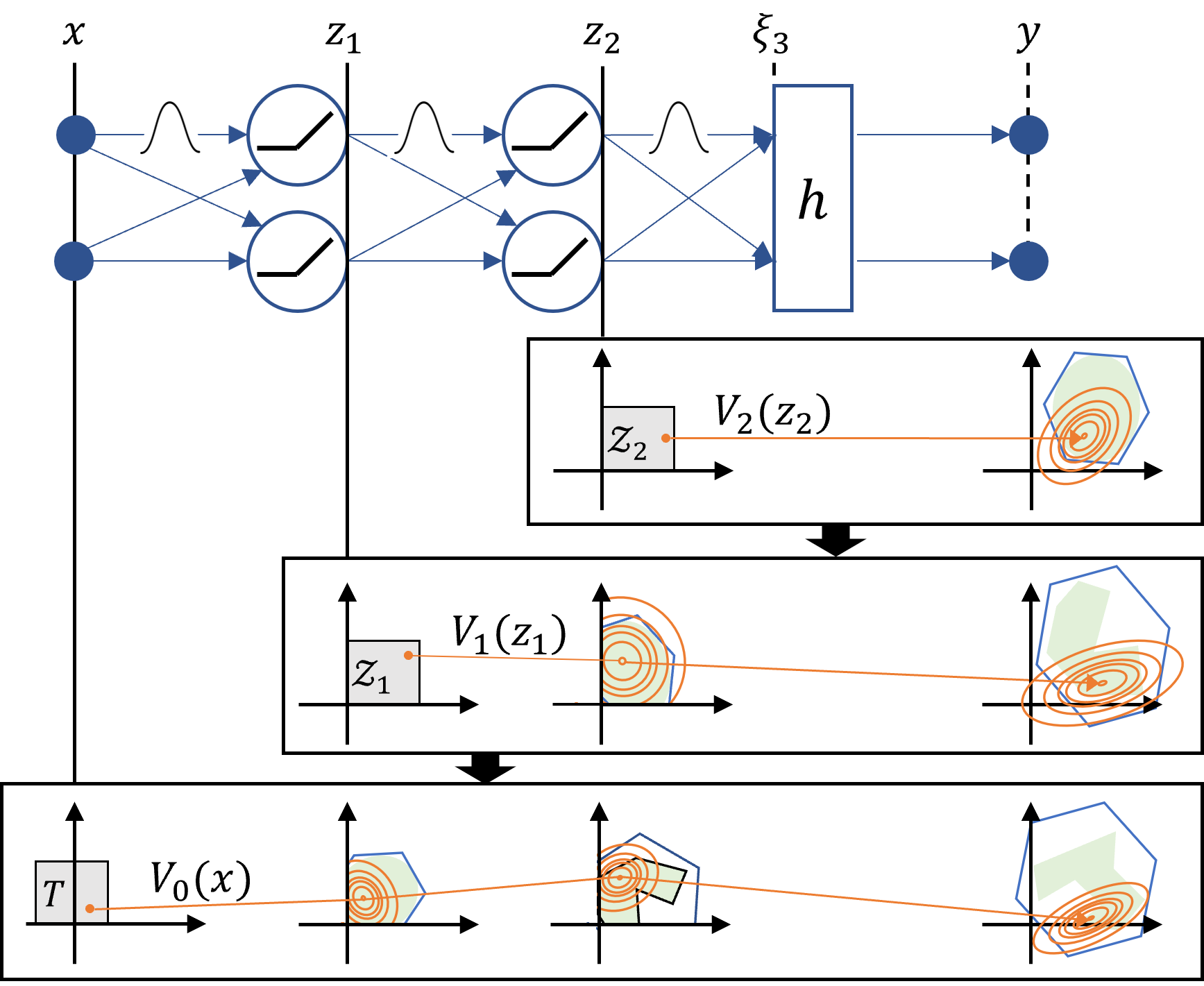}
    % \vspace{-1mm}
    \caption{
    % \ap{In the image I believe that the last $z_3$ should be a pre-activation, ie $\zeta_3$ – really nice image btw}
    Illustration of the DP algorithm in Theorem \ref{Theorem:ValueIteration} for a BNN with two hidden layers. Value functions $V_k$ are mappings from the latent and input spaces of the BNN to the mean of the output distribution. For each mapping, the distribution and mapping of a single point is displayed in orange. Starting from the last hidden layer, we recursively compute PWA approximations of the mappings. 
    The true mean of the BNN for all $z_2\in\mathcal{Z}_2$ is in the green oval, which we over-approximate by the blue hexagon. 
    }
    \label{fig:SDP}
    % \vspace{-2mm}
\end{figure}

\begin{theorem}\label{Theorem:ValueIteration}
Let $\bnn{x}$ be a fully connected BNN with $K$ hidden layers and $h:\mathbb{R}^{n_{K+1}}\to\mathbb{R}^{l}$ be an integrable function. For $k=0,...,K$, define functions $\val{k}{}:\mathbb{R}^{n_{k}}\to \mathbb{R}^{l}$ backwards-recursively as:
%\LL{Deleted comment, you are right about the type :)}
\begin{subequations}
    \begin{align}
        \label{Eqn:FinalConditionValueIteration}
        & \val{K}{\z{}}= \E{\bweight{K}\sim q(\cdot)}{h(\bweight{K} (z^T, 1)^T)}, \\
        \label{Eqn:MiddleCOnditionValueIteration}
        &  \val{k-1}{\z{}}= \E{\bweight{k-1}\sim q(\cdot)}{\val{k}{
        \act{k}{\bweight{k-1}(z^T, 1)^T}}}.
    \end{align}
\end{subequations}
Then, it holds that $\E{\bm{w}\sim q(\cdot)}{h(\bnn{x})} = \val{0}{x}.$
%\begin{equation}
%    \E{\bm{w}\sim q(\cdot)}{h(\bnn{x})} = \val{0}{x}. \label{Eqn:SolutionValueIteration}
%\end{equation}
\end{theorem}
The proof of Theorem \ref{Theorem:ValueIteration} is reported in Appendix \ref{subsec:proofValueIteration} and obtained by induction over the layers of the NN by relying on the law of total expectation and independence of the parameters distribution at different layers.\footnote{\label{footnote:qNonBlockDiag} While the vast majority of VI algorithms make this assumption, Theorem \ref{Theorem:ValueIteration} can be generalized to the case where there is inter-layer correlation by marginalizing Eqn.~\eqref{Eqn:FinalConditionValueIteration} and \eqref{Eqn:MiddleCOnditionValueIteration} over partitions in correlated weight-spaces.}
Figure \ref{fig:SDP} illustrates the backward-iteration scheme of Theorem \ref{Theorem:ValueIteration} for a two-hidden-layer BNN.
% We construct the value function to describe how the output of the last hidden layer is transformed through the last linear layer and function $h$ (Eqn.~\eqref{Eqn:FinalConditionValueIteration}). Then, the value function is propagated backwards describing how the output of the layer $k$ is transformed in the following layers (Eqn.~\eqref{Eqn:MiddleCOnditionValueIteration}). 
Starting from the last layer, value functions $\val{k}{}$ are constructed according to Eqns.~\eqref{Eqn:FinalConditionValueIteration} and \eqref{Eqn:MiddleCOnditionValueIteration} describing how the output of layer $k$ is transformed in the previous layers. 
Theorem~\ref{Theorem:ValueIteration} is a central piece of our framework as it allows one
to break
% on the fact that it allows one to divide 
the computation of $\E{\bm{w}\sim q(\cdot)}{h(\bnn{x})}$ into $K+1$ (simpler) sub-problems, one for each layer of the BNN. 
In fact, note that $\val{k}{}$ is a deterministic function. Hence, all the uncertainty in $\val{k-1}{}$ depends only on the weights of layer $k-1$. This is a key point that we use to derive efficient methods to solve Problem \ref{prob:MainProblem}. 
% compute the extremes of \eqref{Eqn:SolutionValueIteration} for all $x\in T$ and 
% In what follows we will show that in some cases of interest (e.g., BNNs with one hidden layer with ReLU activation functions for regression problems) there exists a closed-form expression of $\val{k}{}$. However,
Nevertheless, we stress that since $\val{k}{z}$ is obtained by propagating $z$ over $K-k$ layers of the BNN, this is still generally a non-convex function, whose exact optimisation is infeasible in practice. 
% \ml{Instead, we focus on bounding these equations? The following corollary ...}
Consequently,  we employ the following corollary, which guarantees that,
to solve Problem \ref{prob:MainProblem},
it suffices to recursively bound $V_k$ following Eqns. \eqref{Eqn:FinalConditionValueIteration} and \eqref{Eqn:MiddleCOnditionValueIteration}.

\begin{corollary}
\label{Corol:MinMaxValFunction}
For $k\in \{1,\hdots, K\}$, let functions $\check{V}_{k}, \hat{V}_{k}: \mathbb{R}^{n_{k}}\rightarrow \mathbb{R}^{n_{l}}$ be relaxations of $\val{k}{\z{k}}$, i.e, $\forall \z{k}\in\mathbb{R}^{n_{k}}, \check{V}_{k}(\z{k})\leq \val{k}{\z{k}}\leq \hat{V}_{k}(\z{k})$. Then 
% \LL{Suggestion. Let's not define $\bz{k}$, but rather let's use $\E{\bweight{k-1}\sim q(\cdot)}{\check{V}_{k}{
%         \act{k}{\bweight{k-1}(z^T, 1)^T}}}.$ This will clarify the distribution wrt we are integrating and will link better to Eqn 8.}
% \begin{align*}
%     \E{}{\check{V}_{k}(\bz{k})}
%     \leq \val{k-1}{\z{k-1}} \leq \E{}{\hat{V}_{k}(\bz{k})}, 
% \end{align*}
\begin{align*}
        \E{\bweight{k-1}\sim q(\cdot)}{\check{V}_{k}(
        \act{k}{\bweight{k-1}(z^T, 1)^T})} \leq \val{k-1}{\z{k}} \leq& \\
        \E{\bweight{k-1}\sim q(\cdot)}{\hat{V}_{k}(
        \act{k}{\bweight{k-1}(z^T, 1)^T})}.& 
\end{align*}
% where $\bz{k}=\act{k}{\bweight{k-1}(\z{k-1}^T, 1)^T}$ and $\bweight{k-1}\sim q(\cdot)$.  
Further, for $i\in\{1,\hdots,l\}$, it holds that  $\pi_{\min}^{(i)}(T) \geq \min_{x \in T} \check{V}^{(i)}_0(x)$ and $  \pi_{\max}^{(i)}(T) \leq \max_{x \in T} \hat{V}^{(i)}_0(x).$
%\begin{align*}
%        \pi_{\min}^{(i)}(T) \geq \min_{x \in T} \check{V}^{(i)}_0(x), \qquad \pi_{\max}^{(i)}(T) \leq \max_{x \in T} \hat{V}^{(i)}_0(x).
%\end{align*}
\end{corollary}
% \LL{Let's be consistent in the paper. Somtimes we go from k+1 to k and sometimes from k to k-1. Let's stick with one for all paper.}
Corollary \ref{Corol:MinMaxValFunction} allows us to recursively find relaxations of $\val{k}{}$ via  Theorem \ref{Theorem:ValueIteration}. In what follows, we focus on finding PWA relaxations $\check{V}_{k}{}$ and $\hat{V}_{k}{}$. To achieve that, there are two basic steps: 
(i) initialization of $\check{V}_{k}{},\hat{V}_{k}{}$ via Eqn.~\eqref{Eqn:FinalConditionValueIteration}, and (ii) backward propagation of $\check{V}_{k}{},\hat{V}_{k}{}$
through a hidden layer of the BNN via Eqn.~\eqref{Eqn:MiddleCOnditionValueIteration}. 
In Section \ref{sec:SolveDynamicProgram}, we first show an efficient method to perform step (ii) and then focus on (i).

\section{PWA Relaxation for Dynamic Programming}\label{sec:SolveDynamicProgram}
%\ml{Title suggestion: PWA Relaxation for Dynamic Programming}
% \section{Solving the Dynamic Program for the Regression Case}\label{sec:SingleLayerBNN}
%IN order to solve Problem \ref{prob:MainProblem} our goal is to minimize and maximize $\val{0}{}$. However, as mentioned in the previous Section, exact optimization of $\val{0}{}$ is in general infeasible. \LL{Considering what we added in the previous Section from here you can delete the rest of this paragraph. I know the discussion about extension to general activation functions we had, but I feel that the sentence you added before of Proposition 2 is more than enough} due to Eqn.~\eqref{Eqn:MiddleCOnditionValueIteration} being highly non-convex. Hence, we employ the result of Corollary \ref{Corol:MinMaxValFunction} and focus on finding convex relaxations of Eqn.~\eqref{Eqn:MiddleCOnditionValueIteration}.
%Our methods rely on convexity properties of Gaussian variables propagated through the ReLU function. Similar to \citep{katz2017reluplex}\SA{add citations} we extend our methods to general continuous activation functions using linear combinations of ReLU functions as relaxations of the activation functions.
Our goal in this section is to find PWA relaxations of $V_k$. To do that, we first show how to propagate affine relaxations of the value function backwards through a single hidden layer of the BNN via Eqn.~\eqref{Eqn:MiddleCOnditionValueIteration} and then generalize this result to PWA relaxations.
Note that, because the support of a BNN is generally unbounded,  affine relaxations  of Eqn.~\eqref{Eqn:MiddleCOnditionValueIteration} and ~\eqref{Eqn:FinalConditionValueIteration} lead to overly conservative results (an affine function should over-approximate a non-linear function over an unbounded set). Thus, PWA relaxations are necessary to obtain tight approximations.
%\LL{IS this really the reason we need piecewise affine? If you explain this already before, you can also delete the rest of this sentence} in order to propagate relaxations of the value function through multiple hidden layers. 
Finally, in Subsection \ref{sec:DecisionLayer}  we show how to compute relaxations for Eqn.~\eqref{Eqn:FinalConditionValueIteration}.

% \begin{figure*}[h]
%     \centering
%     \includegraphics[width=\textwidth]{figures/theory/Example1DPropDist.png}
%     \caption{The pdf of the states of a single dimensional BNN. Up to the second linear layer, there exists an analytical expression for the expected value of the state. \SA{To be updated} \SA{add distributions weight biases} \SA{add ref}}
%     \label{fig:SingleNeuronDist}
% \end{figure*}

\subsection{Affine Value Functions}\label{sec:AffineValFunctions}
% \ap{Have a chat about the presentation of this, using V vs the bound}\ap{Also need to make clear that what we are doing here is the following: we start with an affine relaxation and show how to relax the propagation to have an affine relaxation again. Initially it sounded like we were doing this exactly.
% }
% \LL{Maybe, to make things clearer, we can talk about $\check{V}_{k}$ or $ \hat{V}_{k}$ instead of $V_k.$ This would allow us to link better to Corollary 4.2 and would indeed make clearer how we propagate upper and lower bound functions. 
% }
% \SA{I made several changes to Section 4, and the introduction of Section 5. I emphasized more on us computing relaxations, instead of propagating the true value functions.}
% \ap{I would guide the reviewer as much as possible in this subsection because it is the most heavy notationally and mathemathically. Many won't fully read this section – they'll read ReLU, and will think ``ah-ah you guys have assumptions". We should say here in the intro to the section that we first tackle the ReLU case but then explain how to generalise to the general case with PWA things. 
% }
% \ml{I agree AP's comment.}
For the sake of presentation, we focus on upper bound $\valU{k-1}{}$; the lower bound case follows similarly. 
 Let $\valU{k}{}:\mathbb{R}^{n_k}\rightarrow\mathbb{R}^l$ be an affine upper bound on $\val{k}{}$. Then, by Corollary \ref{Corol:MinMaxValFunction} and the linearity of expectation, it holds that 
\begin{equation}\label{eq:PropLBValElem}
    \valU{k-1}{z}= \valU{k}{\E{\bweight{k-1}\sim q(\cdot)}{\act{k}{\bweight{k-1}(z^T,  1)^T}}}.
\end{equation}
Recall that here $q$ is a Gaussian distribution (see Section \ref{sec:BNNs}). 
Hence, 
 due to the closure of Gaussian random variables w.r.t. linear transformations, we can rewrite Eqn.~\eqref{eq:PropLBValElem} as:
\begin{equation}\label{eq:PropAffineVal}
    \valU{k-1}{\z{}} = \valU{k}{\E{\bprez{}\sim\nDist{m_{k}(\z{})}{\diag{s_{k}(\z{})}}}{\act{k}{\bprez{}}}},
\end{equation}
where $m_k:\mathbb{R}^{n_{k-1}}\rightarrow\mathbb{R}^{n_{k}}$ and $s_k:\mathbb{R}^{n_{k-1}}\rightarrow\mathbb{R}_{\geq 0}^{n_{k}}$ are defined component-wise as 
% \LL{Check type of $s_k$, if it is correct. Also, I could not find where $\mu_{w,k-1,i}(\z{}^T, 1)$ and $\Sigma_{w,k-1,i}$ are defined. \SA{I define them in Eqn \eqref{eq:qForNode}} \LL{I see, cool. Then, instead of defyining again their type again, let's just say "where $\mu_{w,k-1,i},\Sigma_{w,k-1,i}$ are as defined in Eqn 4 and represent meand and covariance of ... " }}
\begin{equation}\label{eq:defFuncs_mAnds}
\begin{aligned}
    &m_{k}^{(i)}(\z{}) = \mu_{w,k-1,i}(\z{}^T, 1)^T,\\
    &s_{k}^{(i)}(\z{}) = (\z{}^T, 1)\Sigma_{w,k-1,i}(\z{}^T, 1)^T, 
\end{aligned}
\end{equation}
for all $i \in \{1,\hdots,n_{k}\}$
with $\mu_{w,k-1,i}, \Sigma_{w,k-1,i}$ being mean and covariance of the $i$th node of the $k$th layer. % (see Eqn.~\ref{eq:qForNode}).
% $\mu_{w,k-1,i}\in\mathbb{R}^{n_{k-1}+1}$ and $\Sigma_{w,k-1,i}\in\mathbb{R}_{\geq 0}^{(n_{k-1}+1) \times (n_{k-1}+1)}$ are the mean and covariance of the $i$th node of the $k$th layer (Eqn.~\eqref{eq:qForNode}), 
$\text{diag}(s)$ is a diagonal matrix with the elements of  $s$ on the main diagonal. 
Note that Eqn~\eqref{eq:PropAffineVal} reduces the propagation of the value function to the propagation of a Gaussian random variable ($\bprez{}$) through an activation 
function ($\phi_{k}$). 
In Proposition \ref{prop:neuron}, we show how this propogation can be achieved analytically for ReLU activiation functions.  Generalization to other activation functions is discussed in Remark \ref{remark:extGenAct}. %, after which we show in Subsection \ref{subsec: pwa value functions}. % how our approach extends to general activation functions. The following proposition shows that we can employ convexity to find affine relaxation of Eqn. \eqref{eq:PropAffineVal} in case $\phi_{k+1}$ is the ReLU function. 
\begin{proposition}\label{prop:neuron}
    For $k\in\{1,\hdots,K\}$, let $\valU{k}{}$ be an affine function and $Z\subset\mathbb{R}^{n_{k-1}}$ be a compact set. Define function
    $r_k:\mathbb{R}^{n_{k-1}}\rightarrow\mathbb{R}^{n_k}_{\geq 0}$ as $r_k(z)=\sqrt{s_{k}(z)}$,
    % \ml{if $r$ is a function, then it needs to take an input. What's the input? $z_{k-1}$?}
    and let $\check{r}_k,\hat{r}_k:\mathbb{R}^{n_{k-1}}\rightarrow\mathbb{R}^{n_k}_{\geq 0}$ be an affine-relaxation 
    of $r_k$ 
    w.r.t.\ $Z$. 
    Further, define $g:\mathbb{R}^2\rightarrow\mathbb{R}$ as
    % $$
    % g(\mu,\sigma) = \frac{\mu}{2}\left[1-\erf{\frac{-\mu}{\sigma\sqrt{2}}} \right] + \frac{\sigma}{\sqrt{2\pi}}\Exp{-\left(\frac{\mu}{2\sigma} \right)^2},
    % $$
    $$
    g(\mu,\sigma) = \frac{\mu}{2}\left[1-\erf{\frac{-\mu}{\sigma\sqrt{2}}} \right] + \frac{\sigma}{\sqrt{2\pi}} \, e^{-(\mu/\sigma\sqrt{2})^2},
    $$
    and, let $\check{g}_i,\hat{g}_i:\mathbb{R}^{2}\rightarrow\mathbb{R}$ be an affine-relaxation of $g$ w.r.t. $\{(m_{k}^{(i)}(\z{}),r_k^{(i)}(\z{}))\mid \forall \z{}\in Z\}$,  Then, for $\matCoefL, \matCoefU\in\mathbb{R}^{n_{k-1}\times n_k}$ and $\vecBiasL, \vecBiasU\in\mathbb{R}^{n_k}$   defined as,
    $\forall i\in \{1,\hdots, n_k\}$,
    \begin{align*}
        &\matCoefL^{(i,:)} = [\nabla_{z} g(m_{k}^{(i)}(\z{}),r_k^{(i)}(\z{}))]_{\z{}=z^*}, \\
        &\vecBiasL^{(i)} = g(m_{k}^{(i)}(z^*),r_k^{(i)}(z^*))  - \matCoefL^{(i,:)}z^*,\\    
        &[\cdot, \matCoefU^{(i,:)} \z{} + \vecBiasU^{(i)}] =(m_{k}^{(i)},\hat{r}_k^{(i)})^T \otimes [\check{g},\hat{g}],
    \end{align*}
    with $z^*\in Z$ and $\nabla_z$ being the gradient w.r.t. $z$, it holds that $\forall \z{}\in Z$, 
    $\E{\bprez{}\sim \nDist{m_{k}(\z{})}{s_{k}(\z{})}}{\valU{k}{\relu{\bprez{}}}} \in \valU{k}{} \otimes [\matCoefL \z{} + \vecBiasL, \matCoefU \z{} + \vecBiasU].$
%    \begin{equation}
 %       \begin{split}
 %           &\E{\bprez{}\sim \nDist{m_{k}(\z{})}{s_{k}(\z{})}}{\valU{k}{\relu{\bprez{}}}} \in \\
 %           &\hspace{4cm}\valU{k}{} \otimes [\matCoefL \z{} + \vecBiasL, \matCoefU \z{} + \vecBiasU].
 %       \end{split}
 %       \label{eq:expectation-RELU}
 %   \end{equation}
\end{proposition}
The proof of Proposition \ref{prop:neuron} is based on the convexity of the expected value of a rectified Gaussian w.r.t.\ its mean and variance.  The proof and detailed procedures for obtaining affine relaxations of $g$ and $r$ are reported in Appendix \ref{subsec:ProofNeuron}. 
Next, we show how the result of Proposition \ref{prop:neuron} can be extended to PWA relaxations of the value functions.
% \LL{Next remark may need to be improved}
% \ml{I gave it a shot}
\begin{remark}\label{remark:extGenAct}
The results of Proposition \ref{prop:neuron} (as well as Propositions \ref{prop:nnCondExpect} and \ref{prop:boundOuterSpace} below) extend to any continuous activation function $\phi_k$. That is, as shown in \citep{benussi2022individual}, 
every continuous activation function can be under and over-approximated by PWA functions $\check{\phi}_k,\hat{\phi}_k:\mathbb{R}^{n_k}\rightarrow\mathbb{R}^{n_k}$ such that $\check{\phi}_k\leq \phi_k \leq \hat{\phi}_k$. Consequently, $
    \mathbb{E}\bigl[\check{\phi}_k(\bprez{}) \bigr] \leq \E{}{\act{k}{\bprez{}}}\leq
    \mathbb{E}\bigl[\hat{\phi}_k(\bprez{}) \bigr]$, which allows the extension of 
    % Eqn.~\eqref{eq:expectation-RELU} 
    Proposition~\ref{prop:neuron}
    from $\text{ReLU}$ to general continuous $\phi_k$.
% In addition, any continuous function can be relaxed using linear combinations of ReLU functions as shown in \SA{\citep{}}. 
%Hence, expectation (conditional) expectations over continuous activation functions can be relaxed computing solely expectations over ReLU functions. 
\end{remark}

% \LL{In the text you said you were going to show or at least discuss how to extend to more general activation functions, but I do not see it in here}

% \subsection{Backward iteration procedure}\label{sec:MultiLayerBNN}
\subsection{Piecewise Affine Value Functions}\label{subsec: pwa value functions}
%In Section \ref{sec:AffineValFunctions} we showed how to compute an affine relaxation of $\val{k}{\z{k}}$ w.r.t. $\z{k}$ being in a compact set $Z_k\subset \mathbb{R}^{n_k}$. However, in practice, BNNs have unbounded support, i.e., for $k\in\{1,...K\},$ $\supp{f^{\bm{w}}_{0:k}}=\mathbb{R}^{n_k}$, where $\supp{f^{\bm{w}}_{0:k}}$ is the support of $f^{\bm{w}}_{0:k}$. In this setting, affine relaxations would in general be overly-conservative. Consequently, in this Section, we focus on finding PWA relaxations of $\val{k}{\z{k}}$. 

For $N\in\mathbb{N}$, let $\mathcal{Z}_k = \{Z_{k,1},\hdots,Z_{k,N}\}\subseteq\mathbb{R}^{n_k}$ 
be a partition of the support of $f^{\bm{w}}_{0:k}$ 
% \LL{No need to define what a partition is. You can delete ", that is $\bigcup_{j=1}^N Z_{k,j} =\supp{f^{\bm{w}}_{0:k}}$ and $\bigcap_{j=1}^N Z_{k,j}=\emptyset$"}
, and let $\check{V}_{k,j}, \hat{V}_{k,j}:\mathbb{R}^{n_k}\rightarrow \mathbb{R}^l$ be an affine relaxation of $\val{k}{\z{k}}$ w.r.t.\ $Z_{k,j}$ for all $j\in\{1,\hdots,N\}$, i.e.,  
$\forall \z{k}\in Z_{k,j}$ $\val{k}{\z{k}} \leq \hat{V}_{k,j}$ with $\hat{V}_{k,j} \coloneqq \matCoefU_{k,j}\z{k}+\vecBiasU_{k,j}.$ 
% $$\val{k}{\z{k}} \leq \hat{V}_{k,j} \coloneqq \matCoefU_{k,j}\z{k}+\vecBiasU_{k,j} \qquad \forall \z{k}\in Z_{k,j}.$$
%$$\hat{V}_{k,j} \coloneqq \matCoefU_{k,j}\z{k}+\vecBiasU_{k,j} \quad \text{and} \quad  \val{k}{\z{k}} \leq \hat{V}_{k,j}.$$
Then, by Eqn.~\eqref{Eqn:MiddleCOnditionValueIteration} and the law of total expectation, we obtain an upper bound on $\val{k-1}{}$:
\begin{multline}
\label{eq:condLBValFunc}
% \begin{aligned}
    \val{k-1}{\z{}} \leq \sum_{j=1}^N \vecBiasU_{k,j}\underbrace{\Prob{\bprez{}\sim \nDist{m_{k}(\z{})}{\diag{s_{k}(\z{})}}}{\bprez{}\in Z_{k,j}}}_{\ref{eq:condLBValFunc} \text{a}} + \\
    \matCoefU_{k,j}\underbrace{\nncondE{\bprez{}\sim \nDist{m_{k}(\z{})}{\diag{s_{k}(\z{})}}}{\act{k}{\bprez{}}}{\bprez{}\in Z_{k,j}}}_{\ref{eq:condLBValFunc}\text{b}}, 
% \end{aligned}
\end{multline}
where $\nncondE{\bprez{}\sim p}{\bprez{}}{\bprez{}\in Z}\coloneqq \condE{\bprez{}\sim p}{\bprez{}}{\bprez{}\in Z}\Prob{\bprez{}\sim p}{\bprez{}\in Z}$.
The lower bound on $\val{k-1}{}$ follows similarly.  
Term \ref{eq:condLBValFunc}a is simply the probability that a Gaussian random variable ($\bprez{}$) is in a given set (partition $Z_{k,j}$). If the partition is hyper-rectangular, in Lemma \ref{lemma:closedFormProbRect} we express Term \ref{eq:condLBValFunc}a in closed-form. 
%\LL{THis before was a proposition, so let's check we refer correctly to it in the paper}
\begin{lemma}\label{lemma:closedFormProbRect}
    % \SA{In Prop \ref{prop:nnCondExpect} we use this result for unbounded intervals, we have a footnote on that. Is there a way to generalize the lemma to this directly? That is, a way to denote an interval that can be closed and open?}
    % \ml{I think the current presentation is clean, but here is what you can do: $\check{\prez{}} ,\hat{\prez{}} \in \mathbb{R}^n \cup \{-\infty,+\infty\}$... and use $\bprez{}\in \langle \check{\prez{}}, \hat{\prez{}} \rangle$, where $\langle \in \big\{(, [\big\}$ and $\rangle \in \big\{), ]\big\}$}
    For $k\in\{1,\hdots,K\}$, $\check{\prez{}}, \hat{\prez{}}\in\mathbb{R}^{n_k}$, it holds that
    \footnote{A similar result holds for unbounded regions defined by vector $\check{z}$, that is, $\bprez{}\in[\check{z},\infty)$ or $\bprez{}\in(\infty, \check{z}]$, as shown in Appendix \ref{sec:ProofClosedFormProbRect}.}
    \begin{align}\label{eq:closedFormProbRect}
    &\Prob{\bprez{}\sim\nDist{m_{k}(\z{})}{\diag{s_{k}(\z{})}}}{\bprez{}\in [\check{\prez{}}, \hat{\prez{}}]} =  \\
    &\frac{1}{2^{n_k}} \prod_{i=1}^{n_k}\erf{\frac{\hat{\prez{}}^{(i)}-m_{k}^{(i)}(\z{})}{\sqrt{2s_{k}^{(i)}(\z{})}}} - 
      \erf{\frac{\check{\prez{}}^{(i)}-m_{k}^{(i)}(\z{})}{\sqrt{2s_{k}^{(i)}(\z{})}}} \nonumber
    \end{align}
\end{lemma}
%The former lemma enables us to find relaxation of Term \ref{eq:condLBValFunc}a by applying interval arithmetic (as explained in the Appendix). 
% For term \ref{eq:condLBValFunc}b, 
Term \ref{eq:condLBValFunc}b is the conditional expectation of the random variable propagated through an activation function. 
The following shows that we can decompose this term in expectations, which we can bound using the result of Proposition \ref{prop:neuron}, and probabilities for which Lemma \ref{lemma:closedFormProbRect} can be applied.
% as in Eqn. \eqref{eq:closedFormProbRect} that can be bounded using Lemma \ref{lemma:closedFormProbRect}.  
% split this term in non-conditional expectations  and probabilities for which 
% The following shows that we can bound this term by using the result of Proposition \ref{prop:neuron} (affine case) and Lemma \ref{lemma:closedFormProbRect}.
% guarantees that to bound this term we can apply the results developed in Section \ref{sec:AffineValFunctions} for the affine case (Proposition \ref{prop:neuron}).
% to bound term \ref{eq:condLBValFunc}b.
\begin{proposition}\label{prop:nnCondExpect}
    For $k\in\{1,\hdots,K\}$, vectors $\check{\prez{}}, \hat{\prez{}}\in\mathbb{R}^{n_k}$, and $\bprez{}\sim\nDist{m_{k}(\z{})}{\diag{s_{k}(\z{})}}$, it holds that\footnote{\label{footnote:gen2identity} A similar relation can be obtained for $\phi_k$ being the identity function, as shown in Appendix \ref{sec:ProofnnCondExpect}.}
    \begin{align*}
        &\nncondE{}{\relu{\bprez{}}}{\bprez{}\in [\check{\prez{}},\hat{\prez{}}]}= \\
        &\qquad\E{}{\relu{\bprez{} + \pospart{\check{\prez{}}}}} - \E{}{\relu{\bprez{} + \pospart{\check{\prez{}}}}} + \\
        &\qquad\quad \pospart{\check{\prez{}}}\Prob{}{\bprez{}\in [\pospart{\check{\prez{}}}, \infty)} -  \pospart{\hat{\prez{}}}\Prob{}{\bprez{}\in [\pospart{\hat{\prez{}}}, \infty)}.
    \end{align*}
\end{proposition}
Next, we show how these results can be extended to unbounded sets in partition $\mathcal{Z}_k$, i.e., unbounded support $f^{\bm{w}}_{0:k}$.

\paragraph*{Unbounded Support}
If $f^{\bm{w}}_{0:k}$ has an unbounded support, then there must necessarily be at least a region that is unbounded in the partition $\mathcal{Z}_k$. While for this region we can still apply Lemma \ref{lemma:closedFormProbRect} to compute Term \ref{eq:condLBValFunc}a, we cannot use Proposition~\ref{prop:nnCondExpect} to compute a bound for Term \ref{eq:condLBValFunc}b. 
Instead, we rely on Proposition~\ref{prop:boundOuterSpace} (below), which derives relaxations based on the fact that Gaussian distributions decay exponentially fast (thus, faster than a linear function).
\begin{proposition}\label{prop:boundOuterSpace}
    For $k\in\{1,\hdots,K\}$, $i\in\{1,\hdots,n_{k}\}$,  and vector $\check{\prez{}}\in\mathbb{R}^{n_k}$, it holds that\footnote{A similar relation can be obtained for $\phi_k$ being the identity function, as shown in Appendix \ref{sec:ProofBoundOuterSpace}.}
    \begin{gather*}
        \new{
        \frac{1}{2}\negpart{m_k^{(i)}(z)}} \leq \\
        \new{\nncondE{\bprez{}\sim\nDist{m_{k}(z)}{\diag{s_{k}(z)}}}{\relu{\bprez{}}}{\bprez{}\in [\check{\prez{}},\infty)}} \\
        \new{\leq  \frac{1}{2}\pospart{m_k^{(i)}(z)} + \sqrt{\frac{s_k^{(i)}(z)}{2\pi}}.}
    \end{gather*}
    %
    %%%%% OLD %%%%%
    % For $k\in\{1,\hdots,K\}$, $i\in\{1,\hdots,n_{k}\}$,  and vector $\check{\prez{}}\in\mathbb{R}^{n_k}$ such that $\check{\prez{}}\in [\pospart{m_{k}(z)}, \infty)$, take $\eta\in\mathbb{R}^{n_k}$ such that $\forall i\in \{1,\hdots, n_k\}$
    % \begin{multline*}
    %     % \eta^{(i)} = &\frac{m_{k}^{(i)}(x)}{2}\left(1 - \erf{\Tilde{\prez{}}^{(i)}} \right) + \\
    %     % &\sqrt{\frac{2s_{k}^{(i)}(x)}{\pi}}e^{-\Tilde{\prez{}}^{(i)} + \frac{1}{2}}\left(\Tilde{\prez{}}^{(i)}+1\right)
    %     \eta^{(i)} = \frac{m_{k}^{(i)}(z)}{2}\left(1 - \erf{\Tilde{\prez{}}^{(i)}} \right) +\\ \left(\Tilde{\prez{}}^{(i)}+1\right)
    %     \sqrt{\frac{2s_{k}^{(i)}(z)}{\pi}} \, e^{-\Tilde{\prez{}}^{(i)} + \frac{1}{2}},
    % \end{multline*}
    % where $\Tilde{\prez{}}^{(i)}=\frac{\check{\prez{}}^{(i)}-m^{(i)}_{k}(z)}{\sqrt{2s_{k}^{(i)}(z)}}$. Then, we have\footref{footnote:gen2identity} $\nncondE{\bprez{}\sim\nDist{m_{k}(z)}{\diag{s_{k}(z)}}}{\relu{\bprez{}}}{\bprez{}\in [\check{\prez{}},\infty)} \leq \eta.$
 %  %  \begin{align*}
 %  %      &\nncondE{\bprez{}\sim\nDist{m_{k}(z)}{\diag{s_{k}(z)}}}{\relu{\bprez{}}}{\bprez{}\in [\check{\prez{}},\infty)} \leq \eta.
 % %   \end{align*}
\end{proposition}

\begin{figure*}[t]
    \centering
    \begin{subfigure}{0.445\textwidth}
        \centering
        \includegraphics[width=0.492\textwidth]{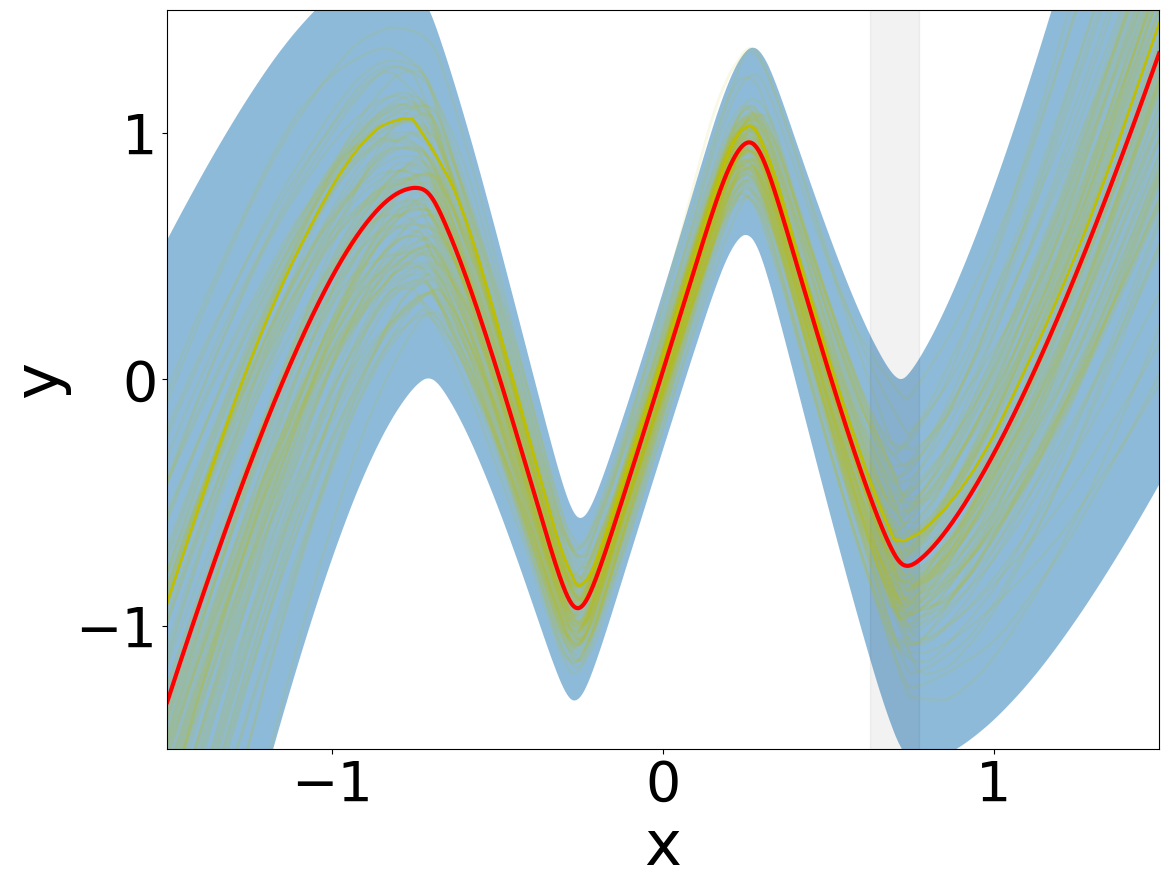}
        \includegraphics[width=0.492\textwidth]{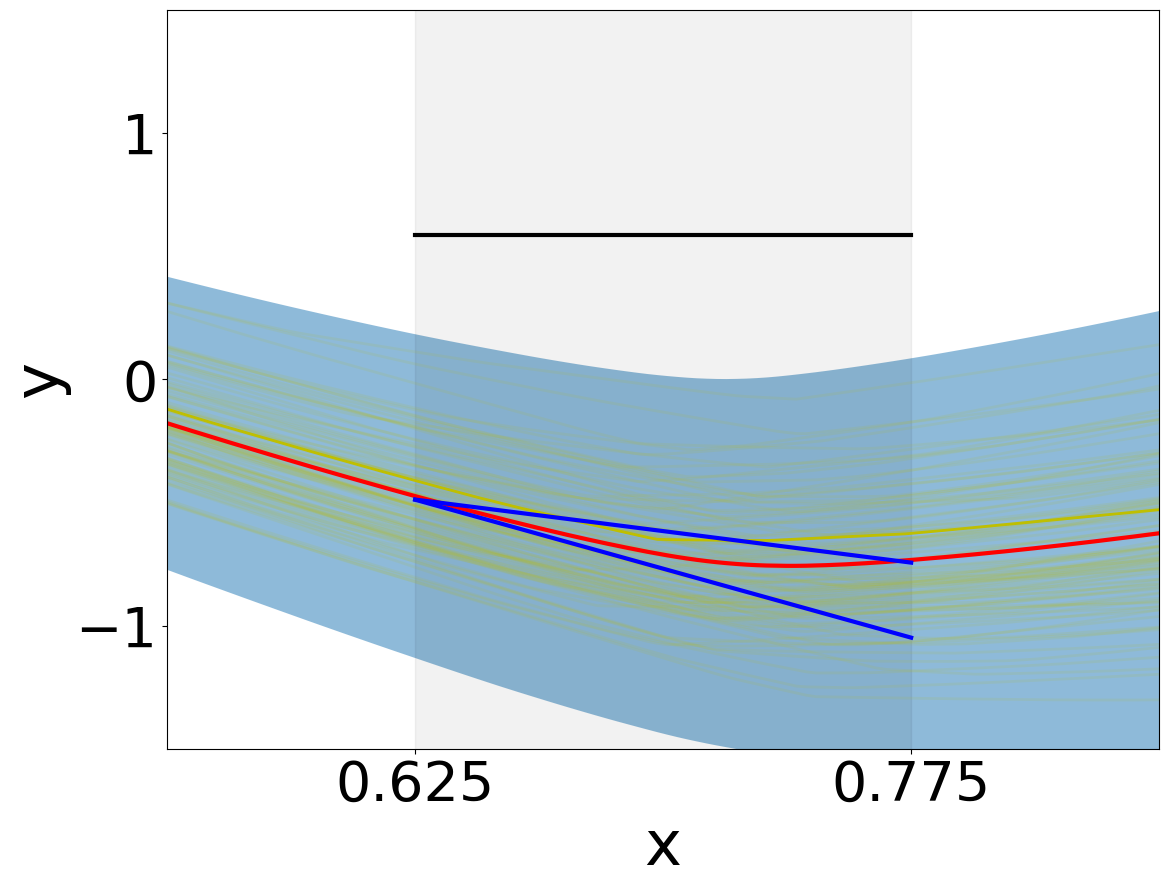}
        % \caption{$K=1, n_{hid}=2048$ and $T=[0.625, 0.775]$}
        \caption{BNN with 1 hidden layer and 2048 nodes. }
        \label{fig:Noisy_sine_1/hid=2048_arch=fc1}
    \end{subfigure}
    \begin{subfigure}[t]{0.1\textwidth}
            \centering
            \includegraphics[width=0.97\textwidth]{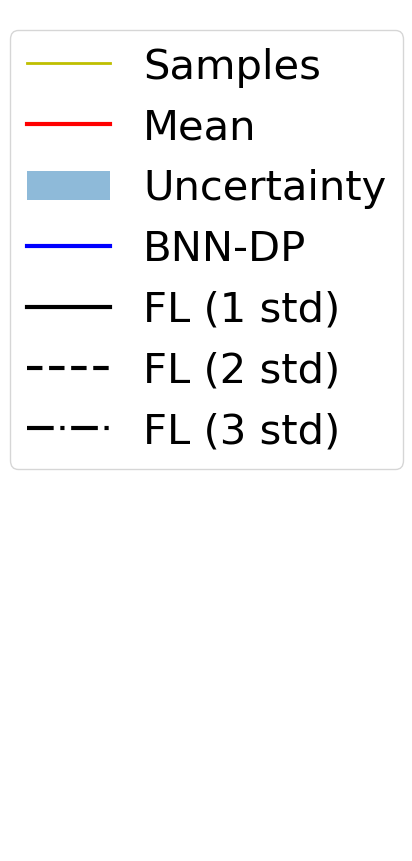}
        \end{subfigure}
    \begin{subfigure}{0.445\textwidth}
        \centering
        \includegraphics[width=0.492\textwidth]{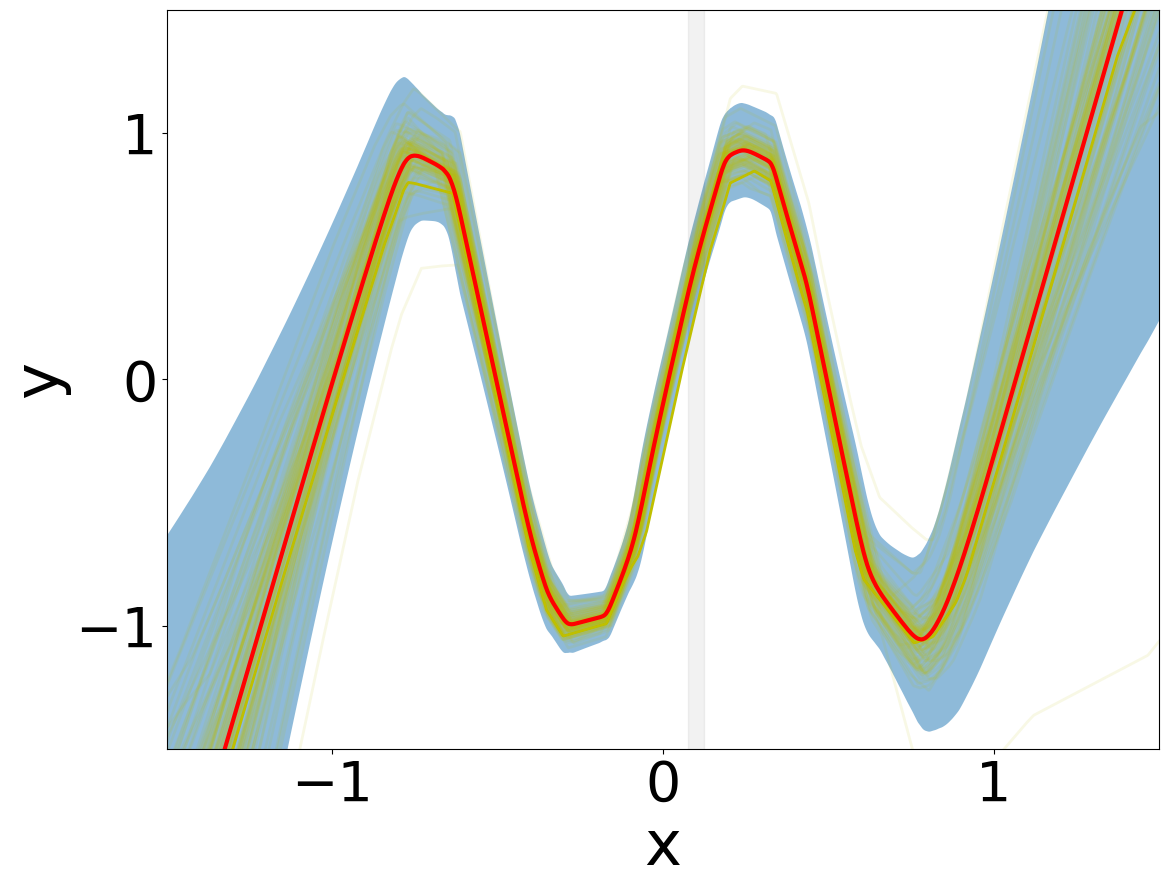}
        \includegraphics[width=0.492\textwidth]{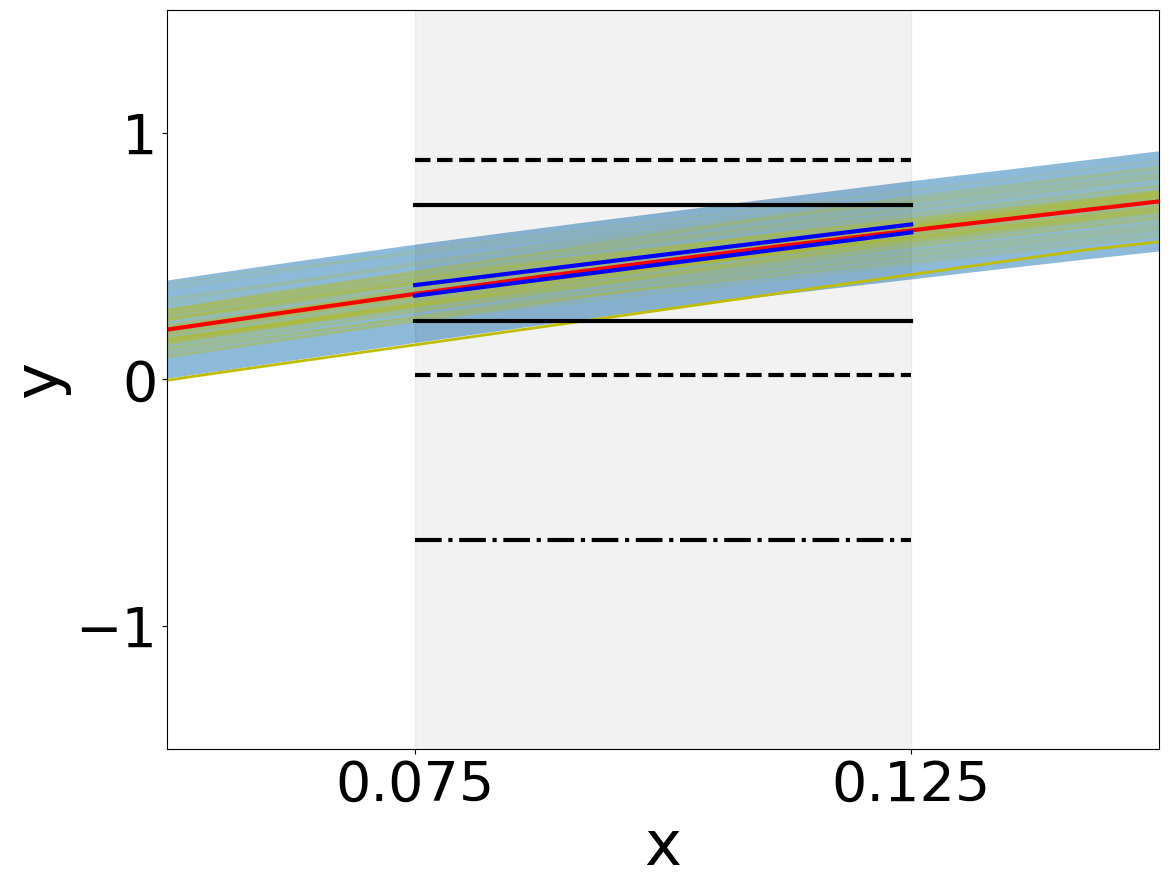}
        % \caption{$K=2, n_{hid}=64$ and $T=[0.075, 0.125]$}
        \caption{BNN with 2 hidden layers and 64 nodes per layer.}
        \label{fig:Noisy_sine_1/hid=64_arch=fc2}
    \end{subfigure}
    \caption{Certified affine bounds on the mean of BNNs trained on the 1D Noisy Sine dataset w.r.t. the grey marked interval of input. 
    % The affine bounds obtained by BNN-DP (blue lines) provide very tight bounds on the sampled mean (red line), while accounting for the full support of the posterior distribution (blue shading). 
    % \ap{The blue shading is not the full posterior distribution, it's probably some standard error around it? For space saving, I don't think you need to explain.\SA{Correct, 3 std}} 
    % Most interval bounds obtained by FL fall out the scope of the axis. 
    }
    \label{fig:noisy_sine_1}
\end{figure*}

 \begin{algorithm}
\caption{Adversarial Robustness for Classification}\label{al:classification}
\begin{algorithmic}[1]
\FUNCTION{Classification($T$, $\{N_k\}_{k=1}^{K+1}$)}
    \FOR{$k\in \{1, \hdots, K+1\}$} 
        \STATE $Z_{k,main} = [\check{\zeta}_k, \hat{\zeta}_k]$ \texttt{(Prop \ref{prop:FindRectForProbBounds})}
        \STATE $\{Z_{k,j}\}_{j=1}^{N_k-1}=\texttt{REFINE}(Z_{k,main})$
        % \STATE $\mathcal{Z}_k = \{Z_{k,j}\}_{j=1}^{N-1}  \bigcup \{\supp{f^{\bm{w}}_{0,k}} - \bigcup_{j=1}^{N-1}Z_{k,j} \}$
        \STATE $\mathcal{Z}_k = \{Z_{k,j}\}_{j=1}^{N_k-1}  \bigcup Z^C_{k,main}$
    \ENDFOR
    \STATE $\{[\check{V}_{K+1,j}, \hat{V}_{K+1,j}]\}_{j=1}^{N_{K+1}} =$ \texttt{IBPSoftmax}($\mathcal{Z}_{K+1}$)
    \FOR{$k\in\{K+1,\hdots,1\}$, \textbf{for} $l\in\{1,\hdots,N_k\}$}
        \STATE $[\check{V}_{k-1,l}, \hat{V}_{k-1,l}] =$ 
        
        \qquad\qquad \texttt{BP}($\{[\check{V}_{k,j}, \hat{V}_{k,j}]\}_{j=1}^{N_{k}}$, $\mathcal{Z}_k$, $\mathcal{Z}_{k-1}^{(l)}$)
    \ENDFOR
    \STATE \textbf{Return:} $\min_{x\in T} \check{V}_{0}(x)$, $\max_{x\in T} \hat{V}_{0}(x)$
\ENDFUNCTION
\end{algorithmic}
\end{algorithm}

% \subsection{Back-Propagation Algorithm}
% \input{sections/sdp/algorithm_regression.tex}

\subsection{Relaxation of the Last Layer of BNN}\label{sec:DecisionLayer} 
% PWA relaxation of the Initialization of the Value Function
% \begin{algorithm}
% \caption{Adversarial Robustness for Classification}\label{al:classification}
% \begin{algorithmic}[1]
% \FUNCTION{Classification($T$, $\{N_k\}_{k=1}^{K+1}$)}
%     \FOR{$k\in \{1, \hdots, K+1\}$} 
%         \STATE $Z_{k,main} = [\check{\zeta}_k, \hat{\zeta}_k]$ \texttt{(Prop \ref{prop:FindRectForProbBounds})}
%         \STATE $\{Z_{k,j}\}_{j=1}^{N_k-1}=\texttt{REFINE}(Z_{k,main})$
%         % \STATE $\mathcal{Z}_k = \{Z_{k,j}\}_{j=1}^{N-1}  \bigcup \{\supp{f^{\bm{w}}_{0,k}} - \bigcup_{j=1}^{N-1}Z_{k,j} \}$
%         \STATE $\mathcal{Z}_k = \{Z_{k,j}\}_{j=1}^{N_k-1}  \bigcup Z^C_{k,main}$
%     \ENDFOR
%     \STATE $\{[\check{V}_{K+1,j}, \hat{V}_{K+1,j}]\}_{j=1}^{N_{K+1}} =$ \texttt{IBPSoftmax}($\mathcal{Z}_{K+1}$)
%     \FOR{$k\in\{K+1,\hdots,1\}$, \textbf{for} $l\in\{1,\hdots,N_k\}$}
%         \STATE $[\check{V}_{k-1,l}, \hat{V}_{k-1,l}] =$ 
        
%         \qquad\qquad \texttt{BP}($\{[\check{V}_{k,j}, \hat{V}_{k,j}]\}_{j=1}^{N_{k}}$, $\mathcal{Z}_k$, $\mathcal{Z}_{k-1}^{(l)}$)
%     \ENDFOR
%     \STATE \textbf{Return:} $\min_{x\in T} \check{V}_{0}(x)$, $\max_{x\in T} \hat{V}_{0}(x)$
% \ENDFUNCTION
% \end{algorithmic}
% \end{algorithm}

We show how to compute interval relaxations of Eqn~\eqref{Eqn:FinalConditionValueIteration}. For the regression case ($h=I$), the process is simple since Eqn. \eqref{Eqn:FinalConditionValueIteration} becomes an affine function. That is, $\val{K}{\z{}}= m_K(\z{})$, where $m_K(\z{})$ as defined in Eqn.~\eqref{eq:defFuncs_mAnds}, and hence no relaxation is required. 
% we directly apply the procedure as described in Section~\ref{sec:AffineValFunctions} to find affine relaxations of 
% % to compute \eqref{Eqn:MiddleCOnditionValueIteration} for $k=K-1$.
% to propagate the value function backwards through the $K$th hidden layer of the BNN. 
For classification, however, further relaxations are needed because $h=\softmax$, i.e.,
% \LL{What you mean by inference?} inference is performed by propagating 
the output distribution of the BNN (the logit) is propagated through the softmax. % (see Section \ref{sec:ProbForm}).
The following proposition shows that an interval relaxation can be obtained by relaxing the distribution of $\bweight{K} (z^T, 1)^T$ by Dirac delta functions on the extremes of $h$ for each set in the partition of the BNN's output. 
%\eqref{Eqn:FinalConditionValueIteration} for $h=\softmax$. %We should stress that, as discussed in \citep{berrada2021make}, for BNNs, differently than for deterministic neural networks \citep{}, we cannot stop at the logit, but we need to explicitly propagate the expectation over the softmax

% \begin{proposition}
%     For hyper-rectangle $Z\in\mathbb{R}^n$ defined by vectors $\check{z},\hat{z}\in\mathbb{R}$, for $i\in\{0,\hdots, n\}$, $\forall z\in Z$ it holds that
%     $$\frac{\exp{\check{z}^{(i)}}}{\exp{\check{z}^{(i)}}+\sum_{j\in\{0,\hdots,i-1,i+1\hdots,n\}}\exp{\hat{z}^{(i)}}} \leq \softmax_i(z) \leq  \frac{\exp{\hat{z}^{(i)}}}{\exp{\hat{z}^{(i)}}+\sum_{j\in\{0,\hdots, i-1,i+1,\hdots,n\}}\exp{\check{z}^{(i)}}}.$$
% \end{proposition}
% https://arxiv.org/ftp/arxiv/papers/1206/1206.6413.
\begin{proposition}\label{prop:decisionLayer}
    For $N\in\mathbb{N}$, let $\{Z_1,\hdots,Z_{N}\}\subseteq\mathbb{R}^{n_{K+1}}$ be a partition of $\supp{\bnn{x}}$. Then, for $i\in\{1,\hdots,n_{K+1}\}$ and $\mathbf{w}\sim q(\cdot)$, it holds that
    \begin{align*}
       % & \sum_{j=1}^N [\min_{\prez{}\in Z_j} h^{(i)}(\prez{})] \Prob{\mathbf{w}\sim q(\cdot)}{\bnn{x}\in Z_j}
       %  \leq \hspace{2cm}\\
       %  & \qquad\quad \E{\mathbf{w}\sim q(\cdot)}{h^{(i)}(\bnn{x}} \leq  \\
       %  &\qquad \qquad\qquad
       %  \sum_{j=1}^N [\max_{\prez{}\in Z_j} h^{(i)}(\prez{})] \Prob{\mathbf{w}\sim q(\cdot)}{\bnn{x}\in Z_j}.
        & \sum_{j=1}^N [\min_{\prez{}\in Z_j} h^{(i)}(\prez{})] \Prob{}{\bnn{x}\in Z_j}
        \leq \E{}{h^{(i)}(\bnn{x}} \leq  \\
        &\qquad \qquad\qquad
        \sum_{j=1}^N [\max_{\prez{}\in Z_j} h^{(i)}(\prez{})] \Prob{}{\bnn{x}\in Z_j}.
    \end{align*}
\end{proposition}
% BNN
% For classification, we can check for
% Recall that, for robust classification, for label $i$, we wish to guarantee that for each class $j$, 
% \ml{unclear}
% \begin{equation}\label{eq:robustClassfication}
%   \E{\bm{w}\sim q(\cdot)}{\softmax^{(j)}(\bnn{}(x))-\softmax^{(i)}(\bnn{}(x))}\leq 0.  
% \end{equation}
A particularly simple case is when there are only two sets in the partition of the BNN's output layer. Then, the following corollary of Proposition~\ref{prop:decisionLayer} guarantees that, similarly to deterministic NNs \citep{zhang2018efficient}, we can determine adversarial robustness by simply looking at the logit. 
\begin{corollary}\label{corol:advRobustnessClass}
    Let $\{[\check{\prez{}},\hat{\prez{}}],Z\}\subseteq\mathbb{R}^{n_{K+1}}$ be a partition of $\supp{\bnn{}}$. Then, for $i,j\in\{1,\hdots,n_{K+1}\}$ and $\bm{w}\sim q(\cdot)$, it holds that
    \begin{align*}
    %      % &(\frac{1}{\Prob{}{\bnn{}(x)\in Z_1}}-1) \sum_{l=1}^{n_{K+1}}\exp(\hat{z}^{(l)}) + \\
    %      % &\qquad \Exp{\hat{z}^{(j)}} - \Exp{\check{z}^{(i)}} \leq 0 \implies\\
         &e^{\hat{\prez{}}^{(j)}} - e^{\check{\prez{}}^{(i)}} + (\frac{1}{\Prob{}{\bnn{}(x)\in [\check{\prez{}},\hat{\prez{}}]}}-1) \sum_{l=1}^{n_{K+1}} e^{\hat{\prez{}}^{(l)}} \leq 0\\
         &\implies \E{}{\softmax^{(j)}(\bnn{}(x))-\softmax^{(i)}(\bnn{}(x))} \leq 0.
    \end{align*}
\end{corollary}
%In what follows we combine the results of this section to a general procedure to solve Problem \ref{prob:MainProblem}.
% \ml{changed $\Exp{blah}$ to $e^{blah}$... Is that OK? If so, we change the other ones as well. \SA{Does it matter if keep using $\Exp{x}$ notation in the appendix? I do need it in there}\LL{For the Appendix you can keep $\Exp{x}$ notation, let's just be coherent in the main text. Do not need to change everywhere in the Appendix}\\
% Does the proof also guarantee that $\Prob{}{\bnn{}(x)\in Z_1} \neq 0$?}
% \SA{$\supp{\bnn{x}} = \mathbb{R}^{n_k}$, and $Z_1$ is assumed compact, so yes, it is guaranteed that $\Prob{}{\bnn{}(x)\in Z_1} \neq 0$}

\section{BNN-DP Algorithm}\label{sec:Algo}
We summarize our overall procedure to solve Problem \ref{prob:MainProblem} in an algorithm called BNN-DP. Algorithm \ref{al:classification} presents BNN-DP for the classification setting; the procedure for the regression setting follows similarly and is provided in Appendix \ref{appen:Algorithms}. 
Algorithm \ref{al:classification} consists of a forward pass to partition the latent space of the BNN (Lines 2-4), and a backward pass to recursively approximate the value functions via Eqns.~\eqref{Eqn:FinalConditionValueIteration} and \eqref{Eqn:MiddleCOnditionValueIteration} (Lines 7-10). 
The last layer of the BNN (Eqn.~\ref{Eqn:FinalConditionValueIteration}) is handled by the \texttt{IBPSoftmax} function in Line 7 using the results of Proposition \ref{prop:decisionLayer}. The \texttt{BP} function in Line 9 performs the back-propagation over the hidden layers of the BNN (Eqn.~\ref{Eqn:MiddleCOnditionValueIteration}) using the results of Lemma \ref{lemma:closedFormProbRect} and Proposition \ref{prop:nnCondExpect} and \ref{prop:boundOuterSpace}. 
The detailed procedures of \texttt{IBPSoftmax} and \texttt{BP} can be found in Appendix \ref{appen:Algorithms}. 
In what follows, we describe how we partition the support of the latent space of the BNN, \new{and discuss the computational complexity of BNN-DP.}

\paragraph{Partitioning}
Recall that our results rely on hyper-rectangular partitions. Hence, 
for each layer $k$, we employ the following proposition to find a hyper-rectangular subset of the support of each layer that captures at least $1 - \epsilon$ of the probability mass of $\supp{f^{\bm{w}}_{0:k}}$. 
\begin{proposition}\label{prop:FindRectForProbBounds}
    For $k\in\{1,\hdots,K\}$,
    let $\epsilon\in[0,1]$ be a constant, and $Z\subset \mathbb{R}^{n_{k-1}}$ be a compact set.
    Then, for vectors $\check{\prez{}}_k,\hat{\prez{}}_k\in\mathbb{R}^{n_k}$ defined such that $\forall i\in\{1,\hdots,n_k\}$,
    \begin{align}
        \check{\prez{}}_k^{(i)} &= \max_{z\in Z} \left[\inverf{-\eta}\sqrt{2s_k^{(i)}(z)}+m_k^{(i)}(z) \right], \label{eq:LowerVecMainPartition} \\
        \hat{\prez{}}_k^{(i)} &= \min_{z\in Z} \left[\inverf{\eta}\sqrt{2s_k^{(i)}(z)}+m_k^{(i)}(z) \right], \label{eq:UpperVecMainPartition}
    \end{align}
    \new{where $\eta = (1-\epsilon)^{\frac{1}{n_k}}$}, 
    it holds that, $\forall \z{}\in Z$,
    $$\Prob{\bprez{}\sim \nDist{m_k(\z{})}{\diag{s_k(\z{})}}}{\bprez{}\in [\check{\prez{}}_k,\hat{\prez{}}_k]} \geq 1-\epsilon.$$
\end{proposition}
\new{Here, Eqns~\eqref{eq:LowerVecMainPartition} and \eqref{eq:UpperVecMainPartition} are convex minimization problems, which can be efficiently solved via, e.g., the gradient descent algorithm.}
We denote the resulting region obtained via Proposition~\ref{prop:FindRectForProbBounds} as $Z_{k,main}\subset \supp{f^{\bm{w}}_{0:k}}$. Then, $Z_{k,main}$ can be further refined by interval splitting.

\paragraph{Computational Complexity}
\new{Similarly as for linear bounding procedures for deterministic neural networks, see e.g. [Zhang et al. 2018], the cost of computing piecewise-affine relaxations of a BNN with $K$ layers and $n$ neurons per layer is polynomial in both $K$ and $n$. 
% Note that BNN-DP has a computational complexity that is polynomial in the number of partitions. 
% In practice, however, in NNs and consequently in BNNs, only few neurons are generally active, and those are the ones that mostly influence the posterior \citep{frankle2018lottery}. Therefore, the refining procedure can focus only on these neurons. In fact, we note that, in almost all the experiments in Section \ref{sec:ExperimentalResults}, only 
% $2$ regions in the partition per hidden layer were required to certify robustness, even in cases where the BNN had large posterior variance and thousands of neurons.
Refinement, which is not part of the main algorithm, has exponential cost in $n$. In practice, however, in NNs and consequently in BNNs, only a few neurons are generally active, and those are the ones that most influence the posterior \citep{frankle2018lottery}. Therefore, the refining procedure can focus only on these neurons. Because of this, in almost all the experiments in Section \ref{sec:ExperimentalResults}, only 
$2$ regions in the partition per hidden layer were required to certify robustness, even in cases where the BNN had large posterior variance and thousands of neurons.}

% \begin{algorithm}[h]
% \caption{Adversarial Robustness for Classification}\label{al:classification}
% \begin{algorithmic}[1]
% \FUNCTION{Classification($T$, $\{N_k\}_{k=1}^{K+1}$)}
%     \FOR{$k\in \{1, \hdots, K+1\}$} 
%         \STATE $Z_{k,main} = [\check{\zeta}_k, \hat{\zeta}_k]$ \texttt{(Prop \ref{prop:FindRectForProbBounds})}
%         \STATE $\{Z_{k,j}\}_{j=1}^{N_k-1}=\texttt{REFINE}(Z_{k,main})$
%         % \STATE $\mathcal{Z}_k = \{Z_{k,j}\}_{j=1}^{N-1}  \bigcup \{\supp{f^{\bm{w}}_{0,k}} - \bigcup_{j=1}^{N-1}Z_{k,j} \}$
%         \STATE $\mathcal{Z}_k = \{Z_{k,j}\}_{j=1}^{N_k-1}  \bigcup Z^C_{k,main}$
%     \ENDFOR
%     \STATE $\{[\check{V}_{K+1,j}, \hat{V}_{K+1,j}]\}_{j=1}^{N_{K+1}} =$ \texttt{IBPSoftmax}($\mathcal{Z}_{K+1}$)
%     \FOR{$k\in\{K+1,\hdots,1\}$, \textbf{for} $l\in\{1,\hdots,N_k\}$}
%         \STATE $[\check{V}_{k-1,l}, \hat{V}_{k-1,l}] =$ 
        
%         \qquad\qquad \texttt{BP}($\{[\check{V}_{k,j}, \hat{V}_{k,j}]\}_{j=1}^{N_{k}}$, $\mathcal{Z}_k$, $\mathcal{Z}_{k-1}^{(l)}$)
%     \ENDFOR
%     \STATE \textbf{Return:} $\min_{x\in T} \check{V}_{0}(x)$, $\max_{x\in T} \hat{V}_{0}(x)$
% \ENDFUNCTION
% \end{algorithmic}
% \end{algorithm}

\section{Experimental Results}\label{sec:ExperimentalResults}
\begin{table*}[t]
    \centering
    \caption{Comparison between BNN-DP and FL on various \new{fully connected} BNN architectures, with $K$ being the number of hidden layers, and $n_{hid}$ the number of neurons per layer.
    The results are the average over $100$ test point, and the computation times are averaged over all architectures.
    The best values for each comparison are reported in bold. 
    \\
    }
    \resizebox{0.50\textwidth}{!}{
    % \begin{scriptsize} % scriptsize
    \subfloat[\Large{$\gamma$-Robustness Regression Tasks}]{
    \begin{tabular}{ccc|rrr|rrr}\toprule
        \multicolumn{3}{c|}{} &\multicolumn{3}{c|}{2D Noisy Sine} &  \multicolumn{3}{c}{Kin8nm} \\ \midrule
        $K$ & $\epsilon$ & $n_{hid}$ & BNN-DP & FL (5 std) & FL (3 std) & BNN-DP  & FL (5 std) & FL (3 std) \\ \midrule
        1 &1e-2   &       64 &  \textbf{0.041} & 1.8 &    0.8 & \textbf{0.044}  &0.7&      0.3 \\
          &       &       256&   \textbf{0.04} & 3.0 &    1.2 & \textbf{0.040}  &45.4&    24.7 \\
          &       &      512 &  \textbf{0.039} & 6.4&    2.5 & \textbf{0.041}  &12.6&     6.0 \\
        2 &1e-3   &       64 &  \textbf{0.109} & 718.1 &  101.1 &\textbf{0.070}  &31.3&    10.8 \\
         &        &      128 &  \textbf{0.239} &112.2&   20.1 & \textbf{0.240}  &1459.8&   420.9 \\
         &        &      256 &  \textbf{0.376} &599.3&   92.1 & \textbf{0.968}  &9420.9&  2715.9 \\
        \new{3} &5e-4   &       64 &  \textbf{0.477} &699.2 & 74.8 &\textbf{0.348}  &12638.5& 59304.6\\
        &        &      128 & \textbf{0.629} & 11214.4 & 1142.8 &\textbf{0.964}  &433149.4& 232811.6 \\
         &        &     256  & \textbf{14.180} & 275408.1 & 2882.3 &\textbf{69.488}  &3441470.8 & 21877545.4 \\ \midrule
        \multicolumn{3}{c|}{Cmp. Time (sec.)} & 8.0 & 7.8   & \textbf{7.7} & 13.2  & \textbf{7.5} & 7.6  \\ \bottomrule
        \label{tab:regression}
    \end{tabular}
    }}
    \resizebox{0.492\textwidth}{!}{
     \subfloat[\Large{$\epsilon$-Robustness Classification Tasks}]{
        \begin{tabular}{cc|ccc|ccc}\toprule
        \multicolumn{2}{c|}{}& \multicolumn{3}{c|}{MNIST}& \multicolumn{3}{c}{Fashion MNIST}\\ \midrule
         $K$ & $n_{hid}$ & BNN-DP & FL (5 std) & FL (3 std)& BNN-DP & FL (5 std) & FL (3 std) \\ \midrule 
            1   &       64  &   \textbf{0.0150} & 0.0090  & 0.0102 & \textbf{0.0128} & 0.0077 &   0.008 \\
                &       128 & \textbf{0.0145} & 0.0091 & 0.0131 & \textbf{0.0065} & 0.0041 & 0.0045\\
                &      256  &  \textbf{0.0137} & 0.0082 &   0.0090 & \textbf{0.0081} & 0.0043 &  0.0046\\
                &      512  & \textbf{0.0131} & 0.0070 & 0.0073 & \textbf{0.0092} & 0.0044 &  0.0048\\
            2   &       64  &  \textbf{0.0073} & 0.0041 &  0.0042 & \textbf{0.0048} & 0.0024 &  0.0026 \\
                &       128 & \textbf{0.0062} & 0.0028 & 0.0035 & \textbf{0.0021} & 0.0018 & 0.0019 \\
                &      256  &  \textbf{0.0049} & 0.0023 &  0.0023 & \textbf{0.0032} & 0.0016 &  0.0016\\
            \new{3}   &      64 &  \textbf{0.0032} & 0.0014 & 0.0016 & \textbf{0.0015} & 0.0006 & 0.0008 \\
                &     256 &  \textbf{0.0018} & 0.0009 & 0.0009 & \textbf{0.0007} & 0.0006 & \textbf{0.0007} \\ \midrule
        \multicolumn{2}{c|}{Cmp. Time (sec)} & \textbf{15.2}  & 859.2 & 805.3 & \textbf{22.1} & 767.6 & 760.2 \\
        \bottomrule
        \label{tab:classification}
    \end{tabular}
    }} 
\end{table*}

We empirically evaluated BNN-DP on various regression and classification benchmarks. We ran our experiments on an AMD EPYC 7252 8-core CPU and train the BNNs using Noisy Adam \citep{zhang2018noisy} and variational online Gauss-Newton \citep{khan2018fast}.
We first validate the bounds obtained by BNN-DP for BNNs trained on samples from an 1D sine with additive noise (referred to as the 1D Noisy Sine). We then analyse a set of BNNs with various architectures trained on the 2D dimensional equivalent of 1D Noisy Sine and the Kin8nm dataset.\footnote{Available at \url{http://www.cs.toronto.edu/~delve}.} The latter dataset contains state-space readings for the dynamics of an 8 link robot arm, and is commonly used as a regression task to benchmark BNNs \citep{hernandez2015probabilistic,gal2016dropout}. % gal2016dropout
Last, we turn our attention to classification and evaluate BNNs trained on the MNIST, Fashion MNIST \new{and CIFAR-10}
datasets. \footnote{\new{Our code is available at \url{https://github.com/sjladams/BNN_DP}.}}

As a baseline for our experiments, we consider the state-of-the-art approach of \citet{berrada2021make}, to which we refer as ``FL''. In fact, FL is the only existing method that can provide robustness certification for BNNs in similar settings as our BNN-DP.
Nevertheless, we must remark that even FL is not fully formal; 
it works by truncating the Gaussian posterior distribution associated to each weight at a given multiple of its standard deviation (std), disregarding a large part of the posterior distribution.  
Hence, the returned bound is not sound over the full posterior but only a subset of it.
More importantly, the disregarded portion of the posterior grows exponentially with the number of weights of the networks. 
Already for a two hidden layer BNN with 48 neurons per layer, FL verifies only $0.1\%$ of the BNN posterior when truncated at 3 std.
Thus, the bounds computed by FL are optimistic and not mathematically guaranteed to hold. 
In contrast, not only BNN-DP returns formal bounds accounting for the whole posterior, but also the benchmark results show that BNN-DP bounds are much tighter than FL ones.

\begin{figure*}[t]
    \centering
    \begin{subfigure}{0.245\textwidth}
        \centering
        \includegraphics[width=0.99\textwidth]{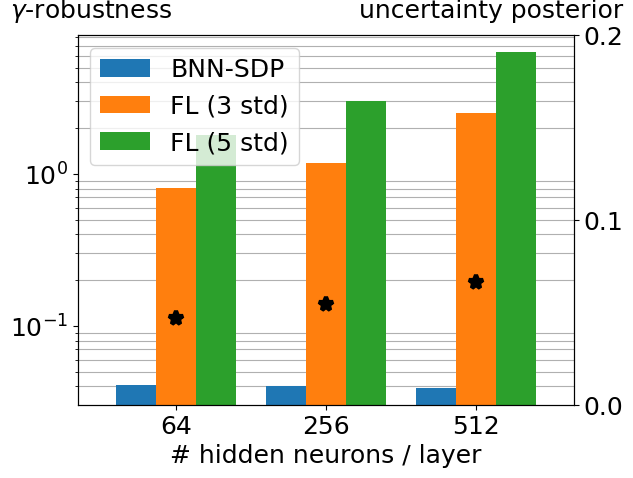}
        \caption{1 hidden layer - 2D NS}
        \label{fig:Noisy_sine_2_fc1}
    \end{subfigure}
    \begin{subfigure}{0.245\textwidth}
        \centering
        \includegraphics[width=0.99\textwidth]{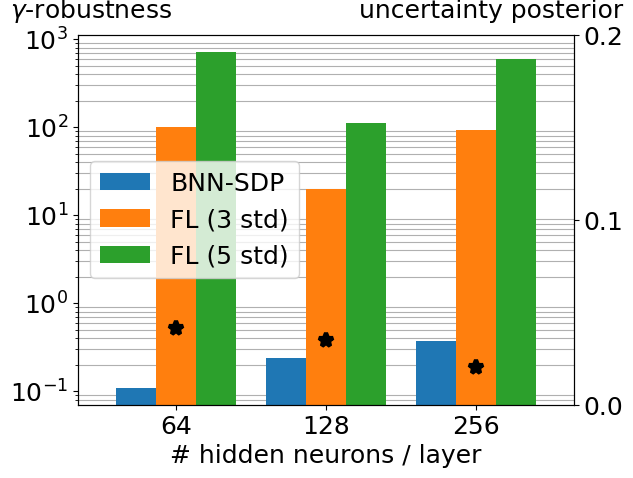}
        \caption{2 hidden layers - 2D NS}
        \label{fig:Noisy_sine_2_fc2}
    \end{subfigure}
    \begin{subfigure}{0.245\textwidth}
        \centering
        \includegraphics[width=0.99\textwidth]{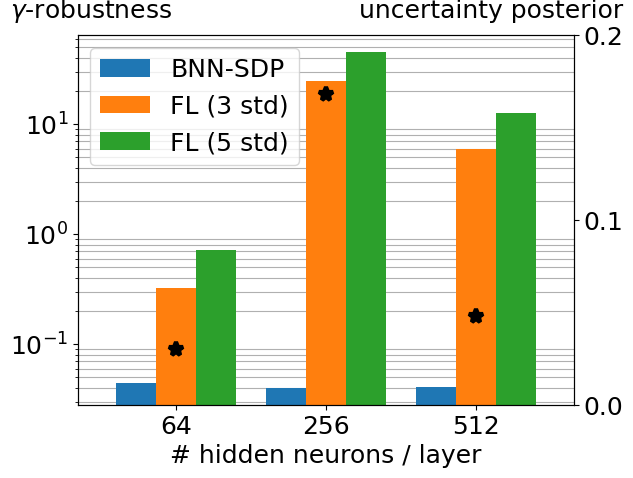}
        \caption{1 hidden layer Kin8nm}
        \label{fig:Kin8nm_fc1}
    \end{subfigure}
    \begin{subfigure}{0.245\textwidth}
        \centering
        \includegraphics[width=0.99\textwidth]{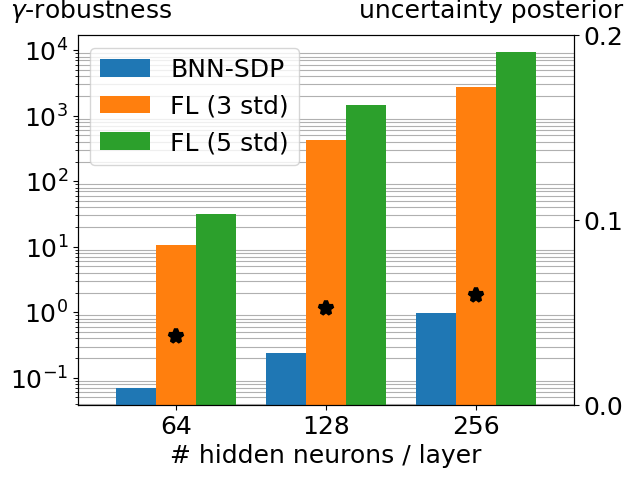}
        \caption{2 hidden layers Kin8nm}
        \label{fig:Kin8nm_fc2}
    \end{subfigure}
    \caption{Analysis of $\gamma$-robustness illustrated as bars (in log-scale), and uncertainty of the posterior distribution (black stars), for BNNs architectures trained on 2D Noisy Sine (NS) and Kin8nm. 
    }
    \label{fig:visualisationTable}
    \ifthenelse{\boolean{arxivFormat}}{\vspace{-2mm}}{}
\end{figure*}
\subsection{Bound Validation}
We validate and qualitatively compare the bounds obtained by BNN-DP and FL on BNNs with 1 and 2 hidden layers trained on 1D Noisy Sine. 
The results of these analyses are reported in Figure \ref{fig:noisy_sine_1}.
Visually, we see that BNN-DP is able to compute tight affine relaxations (blue lines) on the mean of the BNNs over the grey shaded intervals.
In contrast, already in this simple scenario, and even when truncating the posterior distribution at just 1 std, FL returns as guaranteed output intervals $[-1.68, 0.59]$ and $[0.23, 0.71]$ for the 1 and 2 hidden layer BNN, respectively.
Hence, even though FL disregards most of the BNNs posterior, BNN-DP still produces tighter bounds.
When using 3 std, the FL interval bounds become even wider, that is $[-7.07, 9.31]$ and $[-0.65, 1.54]$, for the 1 and 2 hidden layer BNN, respectively.
Intuitively, the major improvement of the bounds can be explained by the fact that, while BNN-DP directly averages the uncertainty of each layer by solving the DP in Theorem \ref{Theorem:ValueIteration}, FL solves an overall optimisation problem that at each layer considers the worst combination of parameters in the support of the (truncated) distribution, leading to conservative bounds. 
In fact, the bound computed by FL is looser in the one-hidden layer case than in the two-hidden layers one by one order of magnitude, precisely because of the higher variance of the former BNN compared to the second. 
In what follows, we see that analogous observations apply to more complex deep learning benchmarks.

\subsection{Regression Benchmarks} 
We consider a set of BNNs with various architectures trained on the 2D Noisy Sine and Kin8nm regression datasets. To asses the certification performance of BNN-DP, we compute the difference between the upper and lower bounds on the expectation of the BNNs, referred to as the $\gamma$-robustness, for its input in a $\ell_{\infty}$-norm ball of radius $\epsilon$ centered at a sampled data point. Clearly, a smaller value of $\gamma$ implies a tighter bound computation.
Results, averaged over 100 randomly sampled test points, are reported in Table~\ref{tab:regression}. For all experiments, BNN-DP greatly improves the value of $\gamma$-robustness provided by the FL-baseline by 1 to 4 orders of magnitude with similar computation times. We also note that the larger the BNN is, the larger the improvement in the (tightness of the) bounds are, which empirically demonstrates the superior scalability of BNN-DP.  
Figure \ref{fig:visualisationTable} explicitly shows the impact of the model size and variance on the certified $\gamma$-robustness. For BNNs with 1 hidden layer, BNN-DP guarantees small $\gamma$-robustness (and hence tighter bounds) irrespective of the number of neurons as well as the amount of uncertainty. In contrast, as already observed for the 1D Noisy Sine case,  
FL is particularly impacted by the variance of the posterior distribution.
For BNNs with two hidden layers, BNN-DP requires partitioning the latent space, which leads to a positive correlation with the value of $\gamma$-robustness and the number of hidden neurons. A similar, but more extreme, trend is also observed for FL.

\newpage
\subsection{Classification Benchmarks}
We now evaluate BNN-DP on the MNIST, Fashion MNIST and \new{CIFAR-10} classification benchmarks. In order to quantitatively measure the robustness of an input point $x^*$, we consider the maximum radius $\epsilon$ for which the decisions on $\ell_{\infty}$-norm perturbations of $x^*$ with radii $\epsilon$ is invariant. That is, any perturbation of 
$x^*$ smaller than $\epsilon$ does not change the classification output; hence, the larger $\epsilon$ in $x^*$, the more robust the BNN in the specific point. 
Results are reported in Table \ref{tab:classification} \new{and \ref{tab:classificationForCNNS}}. 
\new{
For the fully connected BNN architectures, BNN-DP not only is able to certify a substantially larger $\epsilon$ compared to the baseline, but also it does so by orders of magnitude smaller computation time. 
}
% \SA{In all cases, BNN-DP not only is able to certify a substantially larger $\epsilon$ compared to the baseline, but also it does so by orders of magnitude smaller computation time.} 
This is because our approach uses interval relaxations (Proposition \ref{prop:decisionLayer}) to bound the softmax, whereas FL explicitly considers a non-convex optimization problem, which is computationally demanding. 
\new{For the Bayesian CNN architectures, FL is able to certify a slightly larger $\epsilon$, at the costs of magnitudes of orders increase of computation time. 
% Here, the impact of truncating the posterior by FL is visible. 
This can be explained by the decreasing support of the BNN posterior certified by FL for increasing network size, whereas, the $\epsilon$ certified by BNN-DP holds for the whole posterior.}
%Note that 
% Results for more architectures can be found in Appendix \ref{appen:Results}.

\begin{table}[h]
    \centering
    \vspace{-1mm}
    \caption{
    \new{Comparison of the $\epsilon$-robustness obtained with BNN-DP and FL for various BNN architectures, with $K_{conv}$ convolutional layers concatenated to $K$ fully connected hidden layers, and $n_{hid}$ neurons per fully connected layer. The convolutional layers have $n_{kern}$ kernels of size $4\times4$ with stride $1$. Inference on the convolutional and linear layers is performed using Dropout and Bayes by Backprop, respectively. 
    The results are the average over $100$ test points, and the computation times are averaged over all architectures.
    The best values for each comparison are reported in bold. }}
    \vspace{2mm}
    \resizebox{0.5\textwidth}{!}{
    \begin{tabular}{ccccc|ccc}\toprule 
         % \multicolumn{4}{c}{} &\multicolumn{3}{|c}{Fashion MNIST} \\ \midrule
         Dataset &$K_{conv}$ & $n_{kern}$ & $K$ & $n_{hid}$ & BNN-DP & FL (5 std) & FL (3 std)   \\ \midrule
         Fashion MNIST&1 & 2 & 1 & 64 & 0.00065 & 0.00112 & \textbf{0.00122}\\ 
         &2 & 2 & 1 & 64 & 0.00061 & 0.00109 & \textbf{0.00117} \\ \midrule
         &\multicolumn{4}{c|}{Cmp. Time (sec)} & \textbf{3.7} & 545.5 & 319.3 \\ \midrule
         % \multicolumn{4}{c}{} & \multicolumn{3}{|c}{CIFAR-10} \\ \midrule
         % $K_{conv}$ & $n_{kern}$ & $K$ & $n_{hid}$ & BNN-DP & FL (5 std) & FL (3 std) \\ \midrule
         CIFAR-10&2 & 4 & 0 & - & 0.00007 & 0.00009 & \textbf{0.00010}  \\ 
         &3 & 3 & 0 & - & 0.00011 & 0.00019 & \textbf{0.00021} \\ \midrule
         &\multicolumn{4}{c|}{Cmp. Time (sec)} & \textbf{1.8} & 250.7 & 201.2 \\ \midrule
    \end{tabular}
    }
    \label{tab:classificationForCNNS}
\end{table} 

\section{Conclusion}
We introduced BNN-DP, an algorithmic framework to certify adversarial robustness of BNNs. BNN-DP is based on a reformulation of adversarial robustness for BNNs as a solution of a dynamic program, for which efficient relaxations can be derived. Our experiments on multiple datasets for both regression and classification tasks show that our approach greatly outperforms state-of-the-art competitive methods, thus paving the way for applications of  BNNs in safety-critical applications.

\section*{Acknowledgements}
This work was supported in part by the NSF grant 2039062.

\bibliography{sample}
\bibliographystyle{icml2023}

%%%%%%%%%%%%%%%%%%%%%%%%%%%%%%%%%%%%%%%%%%%%%%%%%%%%%%%%%%%%%%%%%%%%%%%%%%%%%%%
%%%%%%%%%%%%%%%%%%%%%%%%%%%%%%%%%%%%%%%%%%%%%%%%%%%%%%%%%%%%%%%%%%%%%%%%%%%%%%%
% APPENDIX
%%%%%%%%%%%%%%%%%%%%%%%%%%%%%%%%%%%%%%%%%%%%%%%%%%%%%%%%%%%%%%%%%%%%%%%%%%%%%%%
%%%%%%%%%%%%%%%%%%%%%%%%%%%%%%%%%%%%%%%%%%%%%%%%%%%%%%%%%%%%%%%%%%%%%%%%%%%%%%%
\newpage
\appendix
\onecolumn
% \input{sections/appendix/nnCondExpect}

% \input{sections/appendix/rectified_gaussian}
% % \subsection*{Proofs Section \ref{sec:SDP}}

\section{Proofs Section \ref{sec:SDP}}\label{appen:ProofsSDPSection}
%% ---- PROOF SDP THEOREM ----
\subsection{Proof Theorem \ref{Theorem:ValueIteration}}\label{subsec:proofValueIteration}
    By the law of total expectation and because of the independence of the weights distribution at different layers it holds that
    \begin{align*}
    \E{\bm{w}\sim q(\cdot)}{h(\bnn{x})} &=\E{\bm{w}\sim q(\cdot)}{ \E{\bweight{K}\sim q(\cdot)}{h(\bweight{K}(\bnnUntil{0:K}{x}^T, 1)^T)}}\\
    &=  \E{\bm{w}\sim q(\cdot)}{\val{K}{\bnnUntil{0:K}{x}}} \\
    &=\E{\bm{w}\sim q(\cdot)}{\E{\bweight{K-1}\sim q(\cdot)}{\val{K}{\act{K}{\bweight{K-1}(\bnnUntil{0:K-1}{x}^T, 1)^T}}}}\\
    &= \E{\bm{w}\sim q(\cdot)}{\val{K-1}{\bnnUntil{0:K-1}{x}}} 
    \end{align*}
    Repeating this procedure backwards over the layers of the neural networks, we obtain $\E{\bm{w}\sim q(\cdot)}{h(\bnn{x})}=\val{0}{x}$.

%% ---- PROOF MAX/MIN COROLLARY ----
\subsection{Corollary \ref{Corol:MinMaxValFunction}}
    If, for $k\in\{1,\hdots, K\}$, we have that $\forall z_k\in \mathbb{R}^{n_k}$, $\check{V}_k(z_k)\leq \val{k}{z_k}\leq \hat{V}_k(z_k)$. Then, for any probability density distribution $p:\mathbb{R}^{n_k}\rightarrow\mathbb{R}_{\geq0}$ it holds that
    $$
        \int_{\supp{p}}\check{V}_{k}(\z{})p(\z{})d\z{} \leq \int_{\supp{p}}\val{k}{\z{}}p(\z{})d\z{} \leq \int_{\supp{p}} \hat{V}_{k}(\z{})p(\z{})d\z{}, 
    $$
    or rewritten in terms of expectations,
    $$
        \E{\bz{}\sim p(\cdot)}{\check{V}_{k}(\bz{})} \leq \E{\bz{}\sim p(\cdot)}{\val{k}{\bz{}}} \leq \E{\bz{}\sim p(\cdot)}{\hat{V}_{k}(\bz{})}.
    $$
    Furthermore, by Theorem \ref{Theorem:ValueIteration}, $\check{V}_0(x)\leq \E{\bm{w}\sim q(\cdot)}{h(\bnn{x})} \leq \hat{V}_0(x)$. Consequently, for $i\in\{1,\hdots,l\}$, it holds that
    \begin{align*}
        \min_{x \in T} \E{\bm{w}\sim q(\cdot)}{h(\bnn{x})} \geq \min_{x \in T} \check{V}^{(i)}_0(x), \qquad \max_{x \in T} \E{\bm{w}\sim q(\cdot)}{h(\bnn{x})} \leq \max_{x \in T} \hat{V}^{(i)}_0(x).
    \end{align*}

\section{Proofs Section \ref{sec:SolveDynamicProgram}}
%% ---- PROOF BOUNDS EXPECTATION SINGLE NEURON ----
For the proof of Proposition \ref{prop:neuron} we rely on some properties of \textit{rectified Gaussian Distributions} \citep{harva2004hierarchical, winn2005variational, socci1997rectified} that we will first introduce below.

In the special case of $\phi_k$ being the ReLU function, the dynamics of the BNN over hidden layer $k$ can be described by the so-called \textit{rectified Gaussian Distribution}, as illustrated in Figure \ref{fig:nDist_vs_rectnDist}, and formally defined as follows. 
\begin{figure}[H]
    \centering
    \includegraphics[width=0.75\textwidth]{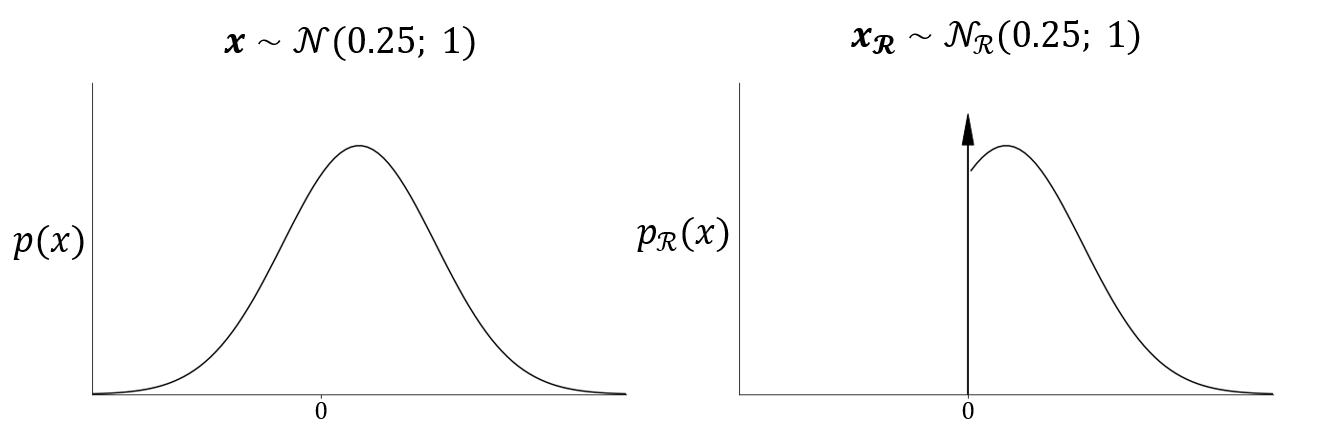}
    \caption{Gaussian probability density function and its related rectified version.}
    \label{fig:nDist_vs_rectnDist}
\end{figure}
\begin{definition}[Rectified Gaussian Distribution]\label{def:rectGaussian}
    Let $\bz{}$ be a random variable with Gaussian distribution $\nDist{\mu}{\sigma^2}$ with mean $\mu\in\mathbb{R}$ and standard deviation $\sigma\in\mathbb{R}_{\geq 0}$. Then, $\max(\bz{},0)$ is a rectified Gaussian random variable. 
\end{definition}
The probability density function (pdf) of rectified Gaussian distributions is obtained by a mixture of a discrete distribution at $\z{}=0$ and a trunctated Gaussian distribution with interval $(0,\infty)$, that is the distribution of $\max(\bz{},0)$ is:
$$ \rectnProbDens{z}{\mu}{\sigma^2} \coloneqq \nCDF{0}{\mu}{\sigma^2}\delta(z) + \nProbDens{z}{\mu}{\sigma^2}U(x), $$
where $\Phi:\mathbb{R}\rightarrow\mathbb{R}_{\geq 0}$ is the cdf of the standard normal distribution, $\delta$ is the Dirac delta function, and $U$ is the unit step function. 

There are various important properties of a rectified Gaussian random variable that we will rely on in this proof. First of all, there exists a closed-form expression for the expected value of rectified Gaussian distributions \citep{harva2004hierarchical}. 
Hence, in the case of $\phi_{k}$ being the ReLU function, we have a closed form-expression for Equation \eqref{eq:PropAffineVal}. Second, as a consequence of the convexity properties of the expectation of rectified Gaussian variables presented in the following two Lemmas, we can employ convexity to efficiently find an affine relaxation of Equation \eqref{eq:PropAffineVal}. 
% \LL{IS this a known property or something we are proving?, \SA{We are proving it.. This property is not reported in the main literature on rectified gaussian I checked. }}

%% ---- LEMMA CONVEXITY EXPECTATION RECTIFIED GAUSSIAN WRT MU & SIGMA ----
\begin{lemma}\label{lemma:convexityExpectRectNormalSigmaMu}
    For $\mu\in\mathbb{R}$ and $\sigma\in\mathbb{R}_{\geq 0}$, it holds that  $\E{\bz{}\sim\rectnDist{\mu}{\sigma^2}}{\bz{}}$
    is convex w.r.t. $(\mu,\sigma)^T$.
\end{lemma}
\begin{proof}
    Let us denote the closed-form expression for the expected value of rectified Gaussian distributions as derived in \citep{harva2004hierarchical} by function  $g:\mathbb{R}^2\rightarrow\mathbb{R}$, that is, we have that, 
    \begin{equation}\label{eq: g}
    \E{\bz{}\sim\rectnDist{\mu}{\sigma^2}}{\bz{}} \coloneqq g(\mu,\sigma) = \frac{\mu}{2}\left[1-\erf{\frac{-\mu}{\sigma\sqrt{2}}} \right] + \frac{\sigma}{\sqrt{2\pi}}\Exp{-\frac{\mu^2}{2\sigma^2}}.
    \end{equation}
    Then, 
    \begin{align*}
    % \frac{\partial g}{\partial \mu} &= \frac{1}{2}\left[1-\erf{-\frac{\mu}{\sigma\sqrt{2}}}\right] \\
    % \frac{\partial g}{\partial \sigma} &=
    % \frac{1}{\sqrt{2\pi}}\Exp{-\frac{\mu^2}{2\sigma^2}} \\
    \frac{\partial^2g}{\partial\mu^2} &= \frac{1}{\sigma\sqrt{2\pi}}\Exp{-\frac{\mu^2}{2\sigma^2}} \nonumber\\
    \frac{\partial^2g}{\partial\sigma^2} &= \frac{\mu^2}{\sigma^3\sqrt{2\pi}}\Exp{-\frac{\mu^2}{2\sigma^2}} \nonumber\\
    \frac{\partial^2g}{\partial\sigma\partial\mu} &= -\frac{\mu}{\sigma^2\sqrt{2\pi}}\Exp{-\frac{\mu^2}{2\sigma^2}} \nonumber
    \end{align*}
    Therefore, the Hessian of $g$ w.r.t. $(\mu,\sigma)^T$, that is, 
    % \LL{Double check derivatives with Mathematica just to be sure there are no typos}
    \begin{align*}
    H_{(\mu,\sigma)} &= \mat{cc}{\frac{\partial^2g}{\partial\mu^2} & \frac{\partial^2g}{\partial\sigma\partial\mu} \\
    \frac{\partial^2g}{\partial\sigma\partial\mu} & \frac{\partial^2g}{\partial\sigma^2}}
    % &= \frac{1}{\sqrt{2\pi\sigma}}\Exp{-\frac{\mu^2}{2\sigma^2}}\mat{cc}{1 & \frac{1}{\sigma}\\ \frac{1}{\sigma} & \frac{1}{\sigma^2}} \\
    % &= \frac{1}{\sqrt{2\pi\sigma}}\Exp{-\frac{\mu^2}{2\sigma^2}}\mat{c}{1 & \frac{1}{\sigma}}\mat{cc}{1 & \frac{1}{\sigma}}
    % &= \frac{1}{\sqrt{2\pi\sigma}}\Exp{-\frac{\mu^2}{2\sigma^2}}cc^T
    \end{align*}
    can be written as
    $$H_{(\mu,\sigma)}= \frac{1}{\sigma\sqrt{2\pi}}\Exp{-\frac{\mu^2}{2\sigma^2}}cc^T$$
    where $c=(1, \frac{1}{\sigma})^T\in\mathbb{R}^2$. To determine whether $H_{(\mu,\sigma)}$ is semipositive definite (spd), that is, whether $\forall z\in\mathbb{R}^2, z^TH_{(\mu,\sigma)}z\geq 0$, first observe that $\forall (\mu,\sigma)\in\mathbb{R}\times\mathbb{R}_{\geq 0}$, we have that $\frac{1}{\sqrt{2\pi\sigma}}\Exp{-\frac{\mu^2}{2\sigma^2}}\geq 0$. Next, we observe that $cc^T$ is spd, because $cc^T$ is a symmetric matrix with all (leading) principal minors equal to zero. Hence, $H_{(\mu,\sigma)}$ is spd and, consequently, $g$ is convex w.r.t. $(\mu,\sigma)$.
\end{proof}

%% ---- LEMMA CONVEXITY EXPECTATION RECTIFIED LINEAR COMBINATION OF GAUSSIANS W.R.T. x ----
\begin{lemma}\label{lemma:convexityExpectRectLinCombNormal}
    Let $m:\mathbb{R}^n\rightarrow\mathbb{R}$ be a linear function defined as $m(x)\coloneqq\mu^T x$, with $\mu\in\mathbb{R}^n$ and $s:\mathbb{R}^n\rightarrow\mathbb{R}$ be a quadratic function $s(x)\coloneqq x^T\Sigma x$ with $\Sigma\in\mathbb{R}^{n\times n}$ a positive definite matrix, then it holds that $\E{\bm{z}\sim\rectnDist{m(x)}{s(x)}}{\bm{z}}$ is convex w.r.t. $x\in\mathbb{R}^n$.
\end{lemma}
\begin{proof}
    Define $g:\mathbb{R}^n\rightarrow\mathbb{R}$ as in Eqn. \eqref{eq: g}. Then, the Hessian of $g$ w.r.t. $x$, that is,
    $H_x=\frac{\partial^2 g}{\partial x^2}$,
    can be written as
    $$H_x = \frac{1}{\sqrt{2\pi s(x)}}(cc^T + \Sigma - dd^T),$$
    where $c = \mu - \frac{m(x)}{s(x)}\Sigma x$ and $d = \frac{1}{\sqrt{s(x)}}\Sigma$. To prove the convexity of $g$ w.r.t. $x$, it suffices to show that matrix $H_x$ is spd for all $x\in \mathbb{R}^n$. We can conclude directly that $\frac{1}{\sqrt{2\pi s(x)}}>0$ and $cc^T$ is a spd matrix, since for any $c\in \mathbb{R}^n$, $cc^T$ is a symmetric matrix with all (leading) principal minors equal to zero. Next, we observe that
    $$xx^T (\Sigma - dd^T) = xx^T \left(\Sigma - \frac{\Sigma xx^T\Sigma}{x^T\Sigma x}\right) = 0_{n,n}$$
    where $0_{n,n}\in\mathbb{R}^{n\times n}$ is the zero matrix. Notice that both matrices $xx^T$ and $0_{n,n}$ are spd. Then, since that the product of two symmetric psd matrices is psd iff its product is symmetric, we conclude that matrix $\Sigma - dd^T$ is psd. Then, since the sum of spd matrices is spd, $H_x$ is spd and consequently $g$ is convex w.r.t. $x$.
    % https://math.stackexchange.com/questions/113842/is-the-product-of-symmetric-positive-semidefinite-matrices-positive-definite
\end{proof}
\subsection{Proof Proposition \ref{prop:neuron}}\label{subsec:ProofNeuron}
Without any loss of generality, we assume that $m:\mathbb{R}^n\rightarrow\mathbb{R}$ is a linear function defined as $m(z)=\mu^T z$, with $\mu\in\mathbb{R}^n$ and $s:\mathbb{R}^n\rightarrow\mathbb{R}$ is a quadratic function $s(z)=z^T\Sigma z$, with $\Sigma\in\mathbb{R}^{n\times n}$ being a positive definite matrix, and define the function $r:\mathbb{R}^n\rightarrow \mathbb{R}$ as $r(z)=\sqrt{s(z)}$.

The expectation of a rectified Gaussian distributed variable is given by function $g: \mathbb{R}^{2}\rightarrow \mathbb{R}^{1}$ as defined in Eqn. \eqref{eq: g}, such that
\begin{align*}
    \E{\bprez{}\sim \rectnDist{m(z)}{s(z)}}{\bprez{}} = g(m(z), r(z)).
\end{align*}
Since $g$ is convex w.r.t. $z$ as shown in Lemma \ref{lemma:convexityExpectRectLinCombNormal}, for $z\in Z$, $g$ can be lower bounded by its tangent at some point $z^*\in Z$, that is, for $\vecCoefL\in \mathbb{R}^n$ and $\biasL\in\mathbb{R}$ defined as 
\begin{align*}
    \vecCoefL &= [\nabla_{z} g(m(z^*), r(z))]_{x=z^*}, \qquad \biasL = g(m(z^*), r(z^*))  - \vecCoefL^T z^*,
\end{align*}
where $\nabla_{z}$ denotes the gradient of $g$ w.r.t.$z$, it holds that $g(m(z), r(z))\geq \vecCoefL^T x + \biasL$. Here, the gradient of $g$ w.r.t. $z$ in terms of $m(z)$ and $r(z)$ is given by
$$
\nabla_{z} g(m(z), r(z)) = \frac{1}{2}\left[1-\erf{\frac{-m(z)}{r(z)\sqrt{2}}} \right]\mu + \frac{\Sigma z}{r(z)\sqrt{2\pi}}\Exp{-\left(\frac{m(z)}{r(z)\sqrt{2}}\right)^2}.
$$
By the convexity of $g$ w.r.t. $z$, we could upper bound $g$ by finding its maximum, located at on the boundary of $Z$, and fitting a hyperplane through the maximum and $n-2$ other points on the edge of $Z$. However, due to the potentially high dimensionality of the $Z$, this direct procedure is infeasible in practice. 

Instead, we first find affine relaxations $\check{g}, \hat{g}$ of $g$ w.r.t. $(m(z),r(z))\in \{(m(z),r(z))\mid \forall z\in Z\}$. After that, we use symbolic arithmetic to propagate affine relaxations $(m(z), \check{r}), (m(z), \hat{r})$ of $(m(z), r(z))$ w.r.t. $z\in Z$, through the symbolic interval $[\check{g}, \hat{g}]$ to obtain an affine relaxation of $g$ w.r.t. $z\in Z$, that is
$$[\cdot, \vecCoefU x + \biasU] =(m^{(i)},\hat{s}^{(i)})^T \otimes [\check{g},\hat{g}].$$
This completes the proof of the proposition. In the remainder, we explain how in practice $\check{g}$, $\hat{g}$ and $\hat{r}$ can be found. 

We denote the set of possible $(m(z), r(z))$ as $P\subset\mathbb{R}\times \mathbb{R}_{\geq}$, that is $P\coloneqq \{(m(z),r(z))\mid \forall z\in Z\}$. To find an affine relaxation of $g$ w.r.t. $(m(z), r(z))\in P$, we use the result of Lemma \ref{lemma:convexityExpectRectLinCombNormal}, that states that $g$ is convex w.r.t. $(\mu,\sigma)$, and for $\check{g}$ take the tangent of $g$ at some point in $P$ and compute $\hat{g}$ by fitting a hyper-plane through the largest $3$ points on the edge of a convex over-approximation of $P$.

To find $\hat{r}$ such that $r(z)\leq \hat{r}(z)$, $\forall z\in Z$, we use symbolic arithmetic to propagate an affine relaxation $\check{s},\hat{s}$ of $s$ w.r.t. $z$ through an affine relaxation $\check{r}_s, \hat{r}_s$ of $r$ w.r.t. $s$. Notice that, since the square root is a strictly increasing function, such that $\hat{r}(z) = \hat{r}_s(\hat{s}(z))$, we only require $\hat{s}$ and $\hat{r}_s$ to find $\hat{r}_s$. As $r$ is concave w.r.t. $s$, we take $\hat{r}_s$ the tangent of $r$ at $s(z^*)$ with $z^*\in X$, that is $\hat{r}_s(z) \leq \frac{s(z)}{2\sqrt{s(z^*)}} + \frac{1}{2}\sqrt{s(z^*)}$. In the case that $\Sigma$ is a diagonal matrix, finding $\hat{s}$ boils down to bounding $n$ one-dimensional quadratic functions. In the case that $\Sigma$ has non-diagonal terms, we first find a transformation matrix $T\in\mathbb{R}^{n\times n}$, such that $T^T\Sigma T$ becomes a diagonal matrix. Then, we bound $s$ in the transformed space induced by $T$, and transform the bounds back to the original space using the inverse transformation to obtain $\hat{s}$. 

\subsection{Proof Lemma \ref{lemma:closedFormProbRect}}\label{sec:ProofClosedFormProbRect}
% \SA{Explain how interval arithmetic can be applied to find interval relaxation.}
    Due to random variable ($\bprez{}$) being independent for each dimension, the computation of the probability that $\bprez{}$ is in  hyper-rectangle $[\check{\prez{}}_k, \hat{\prez{}}_k]$ can be split over the dimension of $\bprez{}$:
    $$
    \Prob{\bprez{}\sim\nDist{m_k(z)}{\diag{s_k(z)}}}{\bprez{}\in [\check{\prez{}}_k, \hat{\prez{}}_k]} =\prod_{i\in \{1,\hdots, n\}} \Prob{\bprezElem{}{i}\sim\nDist{m_k^{(i)}(z)}{s_k^{(i)}(z)}}{\bprezElem{}{i}\in [\check{\prez{}}_k^{(i)}, \hat{\prez{}}_k^{(i)}]}
    $$
    where $\forall i\in \{1,\hdots, n\}$
    \begin{align*}
    & \Prob{\bprezElem{}{i}\sim\nDist{m_k^{(i)}(z)}{s_k^{(i)}(z)}}{\bprezElem{}{i}\in [\check{\prez{}}_k^{(i)}, \hat{\prez{}}_k^{(i)}]} =\\ 
    &\qquad \qquad \frac{1}{2}\left[\erf{\frac{\hat{\prez{}}_k^{(i)}-m_k^{(i)}(z)}{\sqrt{2s_k^{(i)}(z)}}} - \erf{\frac{\check{\prez{}}_k^{(i)}-m_k^{(i)}(z)}{\sqrt{2s_k^{(i)}(z)}}} \right].
    \end{align*}
    Using $\lim_{x\rightarrow\infty}\erf{x}=1$ and $\lim_{x\rightarrow\infty}\erf{-x}=-1$, the above result can be extended to unbounded intervals as follows
    \begin{align*}
    \Prob{\bprezElem{}{i}\sim\nDist{m_k^{(i)}(z)}{s_k^{(i)}(z)}}{\bprezElem{}{i}\in [\check{\prez{}}_k^{(i)}, \infty)} =\frac{1}{2}\left[1 - \erf{\frac{\check{\prez{}}_k^{(i)}-m_k^{(i)}(z)}{\sqrt{2s_k^{(i)}(z)}}} \right].
    \end{align*}

%% ---- PROOF NON-NORMALIZED CONDITIONAL EXPECTATION ----
\subsection{Proof Proposition \ref{prop:nnCondExpect}}\label{sec:ProofnnCondExpect}
    Since the distribution of $\bprez{}$ has a diagonal covariance matrix and $[\check{\prez{}}_k, \hat{\prez{}}_k]$ is a hyper-rectangle, we have that
    \begin{align*}
        \nncondE{}{\act{k}{\bprez{}}}{\bprez{}\in [\check{\prez{}}_k, \hat{\prez{}}_k]}&=\condE{}{\act{k}{\bprez{}}}{\bprez{}\in [\check{\prez{}}_k, \hat{\prez{}}_k]}\Prob{}{\bprez{}\in [\check{\prez{}}_k, \hat{\prez{}}_k]} \\
        &=\int_{\check{\prez{}}_k}^{\hat{\prez{}}_k} \frac{\act{k}{z} \nProbDens{z}{m_{k}}{\diag{s_{k}}}}{\Prob{}{\bprez{}\in [\check{\prez{}}_k, \hat{\prez{}}_k]}} dz \Prob{}{\bprez{}\in [\check{\prez{}}_k, \hat{\prez{}}_k]} \\
        &=\int_{\check{\prez{}}_k}^{\hat{\prez{}}_k} \act{k}{z} \nProbDens{z}{m_{k}}{\diag{s_{k}}}dz\\
        &=\left(\int_{\check{\prez{}}^{(1)}}^{\hat{\prez{}}^{(1)}} \act{k}{z} \nProbDens{z}{m_{k}^{(1)}}{s_{k}^{(1)}}dz, \hdots, \int_{\check{\prez{}}^{(n_k)}}^{\hat{\prez{}}^{(n_k)}} \act{k}{z} \nProbDens{z}{m_{k}^{(n_k)}}{s_{k}^{(n_k)}}dz\right)^T
    \end{align*}
    where we ignore the dependence of $m_{k}$ and $s_{k}$ on $z$ to simplify the notation.
    For $i\in \{1,\hdots, n_k\}$, the integral can be split in two parts:
    \begin{align*}
        \int_{\check{\prez{}}^{(i)}}^{\hat{\prez{}}^{(i)}} \act{k}{z} \nProbDens{z}{m_{k}^{(i)}}{s_{k}^{(i)}}dz &= \int_{\check{\prez{}}^{(i)}}^{\infty} \act{k}{z} \nProbDens{z}{m_{k}^{(i)}}{s_{k}^{(i)}}dz - \int_{\hat{\prez{}}^{(i)}}^{\infty} \act{k}{z} \nProbDens{z}{m_{k}^{(i)}}{s_{k}^{(i)}}dz.
    \end{align*}
    For $\phi_k=I$, using the substitution rule for integration, the former relation can be written as
    \begin{align*}
        &\int_{\check{\prez{}}^{(i)}}^{\hat{\prez{}}^{(i)}} z \nProbDens{z}{m_k^{(i)}}{s_k^{(i)}}dz  = \int_{0}^{\infty} (z+\check{\prez{}}^{(i)}) \nProbDens{z+\check{\prez{}}^{(i)}}{m_k^{(i)}}{s_k^{(i)}}dz - \\
        &\qquad \int_{0}^{\infty} (z+\hat{\prez{}}^{(i)}) \nProbDens{z+\hat{\prez{}}^{(i)}}{m_k^{(i)}}{s_k^{(i)}}dz,
    \end{align*}
    which can be rewritten in terms of expectations and probabilities to obtain the final relation for $\phi_k=I$:
    \begin{align*}
            &\condE{}{\bprezElem{}{i}}{\bprezElem{}{i}\in [\check{\prez{}}^{(i)}, \hat{\prez{}}^{(i)}]}\Prob{}{\bprezElem{}{i}\in [\check{\prez{}}^{(i)}, \hat{\prez{}}^{(i)}]} = \\
            & \qquad\qquad \E{\bz{}\sim\nDist{m_k^{(i)}}{s_k^{(i)}}}{\relu{\bz{}+\check{\prez{}}^{(i)}}} - \E{\bz{}\sim\nDist{m_k^{(i)}}{s_k^{(i)}}}{\relu{\bz{}+\hat{\prez{}}^{(i)}}} + \\
            & \qquad\qquad \check{\prez{}}^{(i)}\Prob{\bz{}\sim\nDist{m_k^{(i)}}{s_k^{(i)}}}{\bz{}\in[\check{\prez{}}^{(i)},\infty]} -  \hat{\prez{}}^{(i)}\Prob{\bz{}\sim\nDist{m_k^{(i)}}{s_k^{(i)}}}{\bz{}\in[\hat{\prez{}}^{(i)},\infty]}.
    \end{align*}
    The relation for $\phi_k=\text{ReLU}$ follows directly from the result for $\phi_k = I$, since $\forall i\in \{1,\hdots, n_k\}$ it hods that
    \begin{align*}
        \int_{\check{\prez{}}^{(i)}}^{\hat{\prez{}}^{(i)}} \relu{z} \nProbDens{z}{m_k^{(i)}}{s_k^{(i)}}dz = \int_{\pospart{\check{\prez{}}}^{(i)}}^{\pospart{\hat{\prez{}}}^{(i)}} z \nProbDens{z}{m_k^{(i)}}{s_k^{(i)}}dz.
    \end{align*}

% \subsection{Lemma \ref{lemma:ConvexityStdLinCombGaussians}}
% \begin{lemma}\label{lemma:ConvexityStdLinCombGaussians}
%     For $\Sigma\in\mathbb{R}^{n\times n}$ a positive definite matrix and function $g:\mathbb{R}^n\rightarrow \mathbb{R}$ defined as $g(x)\coloneqq \sqrt{x^T \Sigma x}$, it holds that $g$ is convex w.r.t. $x$.
% \end{lemma}
% \begin{proof}
%     The Hessian of $g$ w.r.t. $x$, that is $H_x\coloneqq \frac{\partial^2 g}{\partial x^2}$ can be written as 
%     $$H_x = \frac{1}{(x^T\Sigma x)^{\frac{1}{2}}}(\Sigma - \frac{\Sigma xx^T \Sigma}{x^T\Sigma x}).$$
%     To prove the convexity of $g$, it suffices to show that matrix $H_x$ is semi-positive definite (spd) for all $x\in\mathbb{R}^n$. Similar to Lemma \ref{lemma:convexityExpectRectLinCombNormal} we observe that
%     $$ xx^T (\Sigma - \frac{\Sigma x x^T\Sigma}{x^T\Sigma x})=0_{n,n},$$
%     and conclude that since the product of two symmetric psd matrices is psd iff its product is symmetric, $\Sigma - \frac{\Sigma x x^T\Sigma}{x^T\Sigma x}$ is psd. Since in addition $\forall x\in\mathbb{R}^n, \frac{1}{(x^T\Sigma x)^{\frac{1}{2}}}\geq 0$, we conclude that $H_x$ is spd and consequently $g$ is convex w.r.t. $x$.
% \end{proof}

\subsection{Proof Proposition \ref{prop:boundOuterSpace}}\label{sec:ProofBoundOuterSpace}
    Recall from the proof of Proposition \ref{prop:nnCondExpect} that, $\forall i \in \{1,\hdots, n_k\}$, we can write the product of the conditional expectation and probability as
    $$\condE{}{\bprezElem{}{i}}{\bprezElem{}{i}\in [\check{\prez{}}^{(i)},\infty)}\Prob{}{\bprezElem{}{i}\in [\check{\prez{}}^{(i)},\infty)}
        = \int_{\check{\prez{}}^{(i)}}^{\infty}\prez{} \nProbDens{\prez{}}{m_k^{(i)}(z)}{s_k^{(i)}(z)}d\prez{},$$
    % where we ignore the dependence of $m_k$ and $s_k$ on $x$ for simplification of the notation.
    We can apply substitution $q=\frac{\prez{}-m^{(i)}_{k}(z)}{\sqrt{2s_k^{(i)}(z)}}$ to rewrite the integral as
    \begin{align*}
        &\int_{\check{\prez{}}^{(i)}}^{\infty}\prez{} \nProbDens{\prez{}}{m_k^{(i)}(z)}{s_k^{(i)}(z)}d\prez{} = 
        \frac{1}{\sqrt{\pi}}\int_{\Tilde{\prez{}}^{(i)}}^{\infty}(q\sqrt{2s_k^{(i)}(z)}+m_k^{(i)}(z))\Exp{-q^2}dq,
    \end{align*}
    where $\Tilde{\prez{}}^{(i)}(z)=\frac{\check{\prez{}}^{(i)}-m^{(i)}_{k}(z)}{\sqrt{2s_k^{(i)}(z)}}$. 
    We then solve the integral to obtain
    \begin{equation}\label{eq:AnlFormOuterBoundCondExpect}
        \begin{aligned}
            &\int_{\check{\prez{}}^{(i)}}^{\infty}\prez{} \nProbDens{\prez{}}{m_k^{(i)}(z)}{s_k^{(i)}(z)} d\prez{}
            = \frac{m_k^{(i)}(z)}{2}(1 - \erf{\Tilde{\prez{}}^{(i)}(z)}) + \sqrt{\frac{s_k^{(i)}(z)}{2\pi}} \Exp{-(\Tilde{{\prez{}}}^{(i)}(z))^2}.
        \end{aligned}
    \end{equation}
    % %%%% OLD
    % Next, we split the integral, and solve for the $m^{(i)}_{k}$-term, to obtain 
    % \begin{align*}
    %     &\int_{\check{\prez{}}^{(i)}}^{\infty}z \nProbDens{z}{m_k^{(i)}}{s_k^{(i)}} dz
    %     % &\frac{m_k^{(i)}}{\sqrt{\pi}}\int_{\Tilde{\prez{}}^{(i)}}^{\infty}\Exp{-q^2}dq +       
    %     % \sqrt{\frac{2s_k^{(i)}}{\pi}}\int_{\Tilde{\prez{}}^{(i)}}^{\infty} q\Exp{-q^2}dq,\\
    %     = \frac{m_k^{(i)}}{2}(1 - \erf{\Tilde{\prez{}}^{(i)}}) + \sqrt{\frac{2s_k^{(i)}}{\pi}}\int_{\Tilde{\prez{}}^{(i)}}^{\infty}q\Exp{-q^2}dq. 
    % \end{align*}
    % Since $\forall q\in \mathbb{R}$, it holds that $\Exp{-q^2}\leq \Exp{-q+\frac{1}{2}}$. For $\check{\prez{}}^{(i)}-m_k^{(i)} \geq 0$, the integral term can be upper bounded by
    % \begin{align*}
    %     \int_{\Tilde{\prez{}}^{(i)}}^{\infty}q\Exp{-q^2}dq \leq \left[-\Exp{-q+\frac{1}{2}}(q+1)\right]_{\Tilde{\prez{}}^{(i)}}^{\infty}.
    % \end{align*}
    % Applying this upper bound, and using that $\lim_{z\rightarrow\infty}z\Exp{-z}=0$, we obtain
    % \begin{align*}
    %     \int_{\check{\prez{}}^{(i)}}^{\infty}z \nProbDens{z}{m_k^{(i)}}{s_k^{(i)}}dz \leq \frac{m_k^{(i)}}{2}\left(1 - \erf{\Tilde{\prez{}}^{(i)}} \right) + \sqrt{\frac{2s_k^{(i)}}{\pi}}\Exp{-\Tilde{\prez{}}^{(i)} + \frac{1}{2}}\left(\Tilde{\prez{}}^{(i)}+1\right).
    % \end{align*}
    % %%%%
    \new{
    The above result naturally extends to the case of relu-activation functions, which results in 
    \begin{equation}\label{eq:AnlFormOuterBoundCondExpectReLU}
    \begin{aligned}
        &\int_{\check{\prez{}}^{(i)}}^{\infty} \relu{\prez{}} \nProbDens{\prez{}}{m_k^{(i)}(z)}{s_k^{(i)}(z)}d\prez{} = 
        \frac{m_k^{(i)}(z)}{2}(1 - \erf{\pospart{\Tilde{\prez{}}^{(i)}(z)}}) + \sqrt{\frac{s_k^{(i)}(z)}{2\pi}} 
        \Exp{-\pospart{\Tilde{\prez{}}^{(i)}(z)}^2}.
    \end{aligned}
    \end{equation}
    % %%%% OLD
    % Finally, since 
    % $$\int_0^{\infty}z \nProbDens{z}{m_k^{(i)}}{s_k^{(i)}}dz = \int_{0}^{\infty} \relu{z} \nProbDens{z}{m_k^{(i)}}{s_k^{(i)}}dz,$$
    % for $\check{\prez{}}^{(i)} \geq \pospart{m_k^{(i)}}$ it holds that
    % $$\condE{}{\relu{\bprezElem{}{i}}}{\bprezElem{}{i}\in [\check{\prez{}}^{(i)},\infty)}\Prob{}{\bprezElem{}{i}\in [\check{\prez{}}^{(i)},\infty)} \leq \eta^{(i)} $$
    % %%%%
    % Next, we show how to obtain affine relaxations of Eqn.\eqref{eq:AnlFormOuterBoundCondExpect} and hence consequently Eqn.\eqref{eq:AnlFormOuterBoundCondExpectReLU}.
    Since, $\forall x\in\mathbb{R}, 0\leq\erf{\pospart{x}}\leq 1$ and $0\leq \Exp{-x^2}\leq 1$, the closed form solution for the integral can be bounded as follows
    % Eqn.\eqref{eq:AnlFormOuterBoundCondExpect} can bounded as \SA{less conservative lower bound in case of relu-activation function available, in that case $0\leq \erf{\pospart{x}}\leq 1$}
    \begin{equation*}
        \begin{aligned}
            \frac{1}{2}\negpart{m_k^{(i)}(z)} \leq &\int_{\check{\prez{}}^{(i)}}^{\infty}\relu{\prez{}}\nProbDens{\prez{}}{m_k^{(i)}(z)}{s_k^{(i)}(z)} d\prez{} \leq \frac{1}{2}\pospart{m_k^{(i)}(z)} + \sqrt{\frac{s_k^{(i)}(z)}{2\pi}}, \qquad \forall z\in\mathbb{R}^n,
        \end{aligned}
    \end{equation*}
    where, $\sqrt{s_k^{(i)}(z)}$ is convex w.r.t. $z$. Hence, the above upper bound can easily be transformed into a piece-wise affine upper bound.
    % Notice that here, according to Lemma \ref{lemma:ConvexityStdLinCombGaussians}, $\sqrt{s_k^{(i)}(z)}$ is convex. Hence, the above upper bound can easily be transformed into a piece-wise affine upper bound.
    }
    % Hence, $\sqrt{s_k^{(i)}(z)}$ can be upper bounded by its tangent at any point in its domain. That is, 
    % $$\sqrt{s_k^{(i)}(z)} \leq \vecCoefU_s^T(z^T,1)^T + \biasU_s$$
    % where $\vecCoefU_s = 2({\z{}^*}^T, 1) \Sigma_{w,k-1,i}$ and $\biasU_s = \sqrt{s_k^{(i)}(z^*)} - \vecCoefU_s^T ({\z{}^*}^T, 1)^T$ with $z^*\in \mathbb{R}^{n_{k-1}}$.
    % This results in the following affine relaxation of Eqn.\eqref{eq:AnlFormOuterBoundCondExpectReLU},
    % $$ \int_{\check{\prez{}}^{(i)}}^{\infty}\prez{} \nProbDens{\prez{}}{m_k^{(i)}(z)}{s_k^{(i)}(z)} d\prez{} \in [m_k^{(i)}(z), m_k^{(i)}(z) + \frac{1}{\sqrt{2\pi}}( \vecCoefU_s^T(z^T,1)^T + \biasU_s)].$$
    % Remark that if $z\in Z\subset \mathbb{R}^{n_{k-1}}$, we can find an interval relaxation of $\Tilde{\prez{}}^{(i)}(z)$ and improve the upper part of the affine relaxation.  
    %% OLD:
    % $\check{\Tilde{\prez{}}}, \hat{\Tilde{\prez{}}}\in\mathbb{R}$ such that 
    % $\Tilde{\prez{}}^{(i)}(z) \in [\check{\Tilde{\prez{}}}, \hat{\Tilde{\prez{}}}]$ and consequently improve the upper part of the affine relaxation.
    % by taking
    % $$ \int_{\check{\prez{}}^{(i)}}^{\infty}\prez{} \nProbDens{\prez{}}{m_k^{(i)}(z)}{s_k^{(i)}(z)} d\prez{} \leq m_k^{(i)}(z)(1-\erf{}) + \frac{1}{\sqrt{2\pi}}( \vecCoefU_s^T(z^T,1)^T + \biasU_s).$$

\subsection{Proof Proposition \ref{prop:decisionLayer} and Corollary \ref{corol:advRobustnessClass}}
    Let us define function $c:\mathbb{R}^{n_{K+1}}\rightarrow\mathbb{R}$ to simplify notation. By the law of total expectation it holds that
    $$\E{}{c(\bprez{K+1})} = \sum_{j\in\{1,\hdots,N\}} \condE{}{c(\bprez{K+1})}{\bprez{K+1}\in Z_j} \Prob{}{\bprez{K+1}\in Z_j} 
    $$
    where $\bprez{K+1}=\bnn{x}$ and $\bm{w}\sim q(\cdot)$. 
    Here, the conditional expectations can be lower- and upper-bounded by substituting the distribution of random variable $\bprez{K+1}$ by a dirac-delta function placed at the min- and max value of $h^{(i)}$ over $Z_j$, respectively, that is,
    $$
    \E{}{c(\bprez{K+1})} \geq \sum_{j\in\{1,\hdots,N\}} [\min_{\prez{}\in Z_j} c(\prez{})] \Prob{}{\bprez{K+1}\in Z_j} 
    $$
    and 
    $$\E{}{c(\bprez{K+1})} \leq \sum_{j\in\{1,\hdots,N\}} [\max_{\prez{} \in Z_j} c(\prez{})] \Prob{}{\bprez{K+1}\in Z_j}.
    $$
    Hence, if we take $c(\prez{}) = \softmax^{(j)}(\prez{}) - \softmax^{(i)}(\prez{})$, for which holds that
    $$-1\leq \max_{\prez{}\in\mathbb{R}^m}(\softmax^{(j)}(\prez{}) - \softmax^{(i)}(\prez{})) \leq 1,$$
    we have that
    $$
    \E{}{c(\bprez{K+1})} \geq -\Prob{}{\bprez{K+1}\in Z_1} + \sum_{j\in\{2,\hdots,N\}} [\min_{\prez{}\in Z_l} c(\prez{})] \Prob{}{\bprez{K+1}\in Z_j} 
    $$
    and 
    \begin{equation}\label{eq:upBoundClasCond}
    \E{}{c(\bprez{K+1})} \leq \Prob{}{\bprez{K+1}\in Z_1} + \sum_{l\in\{2,\hdots,N\}} [\max_{\prez{}\in Z_j} c(\prez{})] \Prob{}{\bprez{K+1}\in Z_j} 
    \end{equation}
    Then, in the special case of $N=2$, Equation \eqref{eq:upBoundClasCond} provides a sufficient condition to conclude on adversarial robustness. In particular,
    % for $Z_j$ defined by vectors $\check{z},\hat{z}\in\mathbb{R}^{n_{K+1}}$, that is $Z_j=[\check{z},\hat{z}]$, 
    we obtain that
    \begin{align*}
    \max_{\prez{}\in Z_1}\left(\frac{\exp{\prez{}^{(j)}} - \exp{\prez{}^{(i)}}}{\sum_{l\in\{1,\hdots,n_{K+1}\}}\exp(\prez{}^{(l)})}\right) \leq -\frac{\Prob{}{\bprez{K+1}\in Z_0}}{\Prob{}{\bprez{K+1}\in Z_1}}  \implies \E{}{c(\bprez{K+1})} \leq 0.
    \end{align*}
    where we used the definition of the $\softmax$ functions. We can rewrite the former as follows
    \begin{align}\label{eq:proofDecisionLayerCondMiddleStep}
    \max_{\prez{}\in Z_1}\left(\frac{\exp{\prez{}^{(j)}} - \exp{\prez{}^{(i)}}}{\sum_{l\in\{1,\hdots,n_{K+1}\}}\exp(\prez{}^{(l)})}\right) \leq -\eta \implies \E{}{c(\bprez{K+1})} \leq 0.
    \end{align}
    where $\eta \in \mathbb{R}_{\geq 0}$ is defined as $\eta\coloneqq \frac{\Prob{}{\bprez{K+1}\in Z_0}}{\Prob{}{\bprez{K+1}\in Z_1}}{} = \frac{1-\Prob{}{\bprez{K+1}\in Z_1}}{\Prob{}{\bprez{K+1}\in Z_1}}$. Notice that since, $\supp{\bnn{x}} = \mathbb{R}^{n_k}$ and hence $\supp{\bprez{K+1}} = \mathbb{R}^{n_k}$, it is guaranteed that $\Prob{}{\bprez{K+1}\in Z_1} \neq 0$.    
    Since the maximization problem in Eqn.~\eqref{eq:proofDecisionLayerCondMiddleStep} is highly non-convex it can only be solved via exhaustive enumeration as discussed in \citep{berrada2021make}. Hence, to improve computational efficiency, we use that
    $$
    \max_{\prez{}\in Z_1}\left(\frac{\exp{\prez{}^{(j)}} - \exp{\prez{}^{(i)}}}{\sum_{l\in\{1,\hdots,n_{K+1}\}}\exp(\prez{}^{(l)})}\right) \leq \frac{\max_{\prez{}\in Z_1}(\exp{\prez{}^{(j)}} - \exp{\prez{}^{(i)}})}{\max_{\prez{}\in Z_1}(\sum_{l\in\{1,\hdots,m\}}\exp(\prez{}^{(l)}))},
    $$
    to obtain the following condition
    \begin{align*}
    \max_{\prez{}\in Z_1}(\exp{\prez{}^{(j)}} - \exp{\prez{}^{(i)}}) \leq 
     -\eta \max_{\prez{}\in Z_1}\left(\sum_{l\in\{1,\hdots,n_{K+1}\}}\exp(\prez{}^{(l)})\right) \implies \E{}{c(\bprez{K+1})} \leq 0,
    \end{align*}
    which, in the case that $Z_1$ is a hyper-rectangle defined by vectors $\check{\prez{}}, \hat{\prez{}}\in\mathbb{R}^{n_{K+1}}$, reduces to
    \begin{align*}
    \exp{\hat{\prez{}}^{(j)}} - \exp{\check{\prez{}}^{(i)}} +
     \eta \sum_{l\in\{1,\hdots,n_{K+1}\}}\exp(\hat{\prez{}}^{(l)}) \leq 0 \implies \E{}{c(\bprez{K+1})} \leq 0.
    \end{align*}

    % \SA{For truncation:}
    % $$
    % \min_{z\in \supp{\bnn{}}} c(z) \leq \E{}{c(\bprez{K+1})} \leq \max_{z\in \supp{\bnn{}}} c(z)
    % $$
    % Such that
    % $$
    % \max_{z\in \supp{\bnn{}}} c(z) \leq 0 \implies  \E{}{c(\bprez{K+1})} \leq 0
    % $$
    % Take $c(z) = \softmax^{(j)}(z) - \softmax^{(i)}(z)$, then
    % $$
    % \max_{z\in \supp{\bnn{}}} \softmax^{(j)}(z) - \softmax^{(i)}(z) \leq 0 \implies  \E{}{c(\bprez{K+1})} \leq 0
    % $$
    % which can be rewritten as
    % $$
    % \max_{z\in \supp{\bnn{}}} \left( \frac{\exp{z^{(j)}} - \exp{z^{(i)}}}{\sum_{l\in\{1,\hdots,n_{K+1}\}}\exp(z^{(l)})}\right) \leq 0 \implies  \E{}{c(\bprez{K+1})} \leq 0
    % $$
    % which reduces to
    % $$
    % \max_{z\in \supp{\bnn{}}} \left(\exp{z^{(j)}} - \exp{z^{(i)}}\right) \leq 0 \implies  \E{}{c(\bprez{K+1})} \leq 0
    % $$

\section{Proofs Section \ref{sec:Algo}}
\subsection{Proof Proposition \ref{prop:FindRectForProbBounds}}
    Recall that for $[\check{\prez{}},\hat{\prez{}}]$ a hyper-rectangle, according to Lemma \ref{lemma:closedFormProbRect}, computing the probability that $\bprez{}$ is in $Z_k$ reduces to computing the product of Gaussian CDFs (error functions), that is
    \begin{align*}
        \Prob{\bprez{}\sim \nDist{m_k(x)}{\diag{s_k(x)}}}{\bprez{}\in [\check{\prez{}}_k,\hat{\prez{}}_k]} =\prod_{i\in\{1,\hdots,n_k\}} \frac{1}{2}\left[\erf{\frac{\hat{\prez{}}^{(i)}-m_k^{(i)}(x)}{\sqrt{2s_k^{(i)}(x)}}} - \erf{\frac{\check{\prez{}}^{(i)}-m_k^{(i)}(x)}{\sqrt{2s_k^{(i)}(x)}}} \right].
    \end{align*}
    To ensure that $\Prob{\bprez{}\sim \nDist{m(x)}{\diag{s(x)}}}{\bprez{}\in Z_k}= 1-\epsilon$, we enforce that for each $i\in \{1,\hdots,l\}$, it holds that
    $$\frac{1}{2}\left[\erf{\frac{\hat{\prez{}}^{(i)}-m_k^{(i)}(x)}{\sqrt{2s_k^{(i)}(x)}}} - \erf{\frac{\check{\prez{}}^{(i)}-m_k^{(i)}(x)}{\sqrt{2s_k^{(i)}(x)}}} \right] = (1-\epsilon)^{\frac{1}{n_k}}.$$
    Next, we choose to place $\hat{\prez{}}$ and $\check{\prez{}}$ symmetrical around $m(x)$, that is, $\forall i \in\{1,\hdots, n_k\}$, we choose $\hat{\prez{}}^{(i)}$ and $\check{\prez{}}^{(i)}$ such that
    $$\erf{\frac{\hat{\prez{}}^{(i)}-m_k^{(i)}(x)}{\sqrt{2s_k^{(i)}(x)}}} = (1-\epsilon)^{\frac{1}{n_k}}, \quad \text{and} \quad \erf{\frac{\check{\prez{}}^{(i)}-m_k^{(i)}(x)}{\sqrt{2s_k^{(i)}(x)}}} = -(1-\epsilon)^{\frac{1}{n_k}}.$$
    Then, by taking the inverse of the error function for both qualities, we obtain 
    the following expressions for $\check{\prez{}}$ and $\hat{\prez{}}$:
    \begin{align*}
        \hat{\prez{}}^{(i)} &= \inverf{(1-\epsilon)^{\frac{1}{n}}}\sqrt{2s_k^{(i)}(x)}+m_k^{(i)}(x),\\
        \check{\prez{}}^{(i)} &= \inverf{-(1-\epsilon)^{\frac{1}{n}}}\sqrt{2s_k^{(i)}(x)}+m_k^{(i)}(x).
    \end{align*}
    Notice that, these expressions depend via $s(x)$ and $m(x)$ on $x$ which can take values in $X$. Hence, to ensure that $\forall x\in T$, $\Prob{\bprez{}\sim \nDist{m(x)}{\diag{s(x)}}}{\bprez{}\in Z_k} \geq 1-\epsilon$, it should hold that
    \begin{align*}
        \hat{\prez{}}^{(i)} &\geq \min_{x\in X} \left[\inverf{(1-\epsilon)^{\frac{1}{n}}}\sqrt{2s_k^{(i)}(x)}+m_k^{(i)}(x) \right],\\
        \check{\prez{}}^{(i)} &\leq \max_{x\in X} \left[\inverf{-(1-\epsilon)^{\frac{1}{n}}}\sqrt{2s_k^{(i)}(x)}+m_k^{(i)}(x) \right].
    \end{align*}
    Hence, the optimal choice for $\check{\prez{}}^{(i)}$ and $\check{\prez{}}^{(i)}$ is such that above constraints hold by equality. Notice that for $\epsilon \in [0,1]$, both optimization problems reduce to convex \new{minimization} problems.

\section{Algorithms}\label{appen:Algorithms}
The interval relaxation procedure of the softmax layer is summarized in Algorithm \ref{al:softmaxIBP}. The algorithm employs the result of Proposition \ref{prop:decisionLayer}. The back-propagation of a PWA relaxation of the value function is summarized in Algorithm \ref{al:backpropstep}, to which we refer as \texttt{BP}. For partition $\{Z_{k,j}\}_{j=1}^N$ of $\supp{f^{\bm{w}}_{0:k}}$ and a compact subset $Z_{k-1}$ of $\supp{f^{\bm{w}}_{0:k-1}}$, \texttt{BP} computes the relaxations of Terms \ref{eq:condLBValFunc}a in Line 4 and \ref{eq:condLBValFunc}b in Line 6 w.r.t. $z_{k-1}\in Z_{k-1}$. It employs Lemma \ref{lemma:closedFormProbRect} for Term \ref{eq:condLBValFunc}a, and  Proposition \ref{prop:nnCondExpect} or \ref{prop:boundOuterSpace} for Term \ref{eq:condLBValFunc}b. Finally, the relaxations of Terms \ref{eq:condLBValFunc}a and \ref{eq:condLBValFunc}b are combined with PWA relaxation of $\val{k}{}$ w.r.t. partition $\{Z_{k,j}\}_{j=1}^N$, denoted as $\{\check{V}_{k,j},\hat{V}_{k,j}\}_{j=1}^N$, following Eqn.~\eqref{eq:condLBValFunc} in Line 8-9. Lastly, Algorithm \ref{al:regression} presents BNN-DP for the regression setting. 

\begin{algorithm}
\caption{Interval relaxation procedure for Eqn. \ref{Eqn:FinalConditionValueIteration} with $h=\softmax$}\label{al:softmaxIBP}
\begin{algorithmic}[1]
\FUNCTION{IBPSoftmax($\{Z_{j}\}_{j=1}^N$)}
    \FOR{$j\in\{1,\hdots,N\}$}
        \IF{$Z_j$ is a hyperrectangle}
            \STATE Initialize $\check{z},\hat{z}$ such that  $Z_j=[\check{z}^{(0)},\hat{z}^{(0)}]\times\hdots\times[\check{z}^{(n)},\hat{z}^{(n)}]$
            \FOR{$i\in\{1,\hdots,n\}$}
                \STATE $\biasL_j^{(i)} = \frac{\exp{\check{z}^{(i)}}}{\exp{\check{z}^{(i)}}+\sum_{l=1, l\neq i}^N\exp{\hat{z}^{(l)}}}$,
                \STATE 
                $\biasU_j^{(i)} = \frac{\exp{\hat{z}^{(i)}}}{\exp{\hat{z}^{(i)}}+\sum_{l=1, l\neq j}^N \exp{\check{z}^{(l)}}}$
            \ENDFOR
        \ELSE
            \STATE $\biasL_j=0, \biasU_j=1$
        \ENDIF
    \ENDFOR
    \STATE \textbf{Return:} $\{[\biasL_j, \biasU_j]\}_{j=1}^N$
\ENDFUNCTION
\end{algorithmic}
\end{algorithm}

\begin{algorithm}[h]
\caption{Back-Propagation of PWA Relaxations.}
\label{al:backpropstep}
\begin{algorithmic}[1]
    \FUNCTION{\texttt{BP}($\{[\check{V}_{k,j}, \hat{V}_{k,j}]\}_{j=1}^N$, $\{Z_{k,j}\}_{j=1}^N$, $Z_{k-1}$)}
    \FOR{$j\in \{1,\hdots,N\}$}
        \STATE For $\z{}\in Z_{k,j}$, compute $\check{\mathbb{P}}_j, \hat{\mathbb{P}}_j\in [0,1]$ s.t.% (Lemma \ref{lemma:closedFormProbRect}) 
        \STATE $\Prob{\bprez{}\sim\nDist{m_{k}(z)}{s_{k}(z)}}{\bprez{}\in Z_{k,j}}\in [\check{\mathbb{P}}_j, \hat{\mathbb{P}}_j]$, \label{all:backpropstep_termA}
        \STATE compute $\check{\mathbb{E}}_j, \hat{\mathbb{E}}_j:\mathbb{R}^{n_{k-1}}\rightarrow\mathbb{R}^l$ s.t.% (Prop. \ref{prop:neuron}-\ref{prop:boundOuterSpace}) 
        \STATE $\nncondE{\bprez{}\sim\nDist{m_{k}(z)}{s_{k}(z)}}{\act{k}{\bprez{}}}{\bprez{}\in Z_{k,j}} \in [\check{\mathbb{E}}_j, \hat{\mathbb{E}}_j]$ \label{all:backpropstep_termB}
    \ENDFOR
    \STATE $[\check{V}_{k-1}, \cdot] = \sum_{j=1}^N \matCoefL_{k,j} \otimes [\check{\mathbb{E}}_j, \hat{\mathbb{E}}_j] + \vecBiasL_{k,j} \otimes [\check{\mathbb{P}}_{j}, \hat{\mathbb{P}}_{j}] $ \label{all:backpropstep_CombineL}
    \STATE $[\cdot, \hat{V}_{k-1}] = \sum_{j=1}^N \matCoefU_{k,j} \otimes [\check{\mathbb{E}}_j, \hat{\mathbb{E}}_j] + \vecBiasU_{k,j} \otimes [\check{\mathbb{P}}_{j}, \hat{\mathbb{P}}_{j}] $ \label{all:backpropstep_CombineU}
    % \STATE \textbf{Return} $[\check{V}_{k-1}, \hat{V}_{k-1}]$
    \ENDFUNCTION
\end{algorithmic}
\end{algorithm}

\begin{algorithm}[h]
\caption{Adversarial Robustness for Regression}\label{al:regression}
\begin{algorithmic}[1]
\FUNCTION{Regression($T$, $\{N_k\}_{k=1}^{K+1}$)}
    \FOR{$k\in \{1, \hdots, K-1\}$} 
        \STATE $Z_{k,main} = [\check{\prez{k}}, \hat{\prez{k}}]$ \texttt{(Prop \ref{prop:FindRectForProbBounds})}
        \STATE $\{Z_{k,j}\}_{j=1}^{N_k-1}=\texttt{REFINE}(Z_{k,main})$
        % \STATE $\mathcal{Z}_k = \{Z_{k,j}\}_{j=1}^{N-1}  \bigcup \{\supp{f^{\bm{w}}_{0,k}} - \bigcup_{j=1}^{N-1}Z_{k,j} \}$
        \STATE $\mathcal{Z}_k = \{Z_{k,j}\}_{j=1}^{N_k-1}  \bigcup Z^C_{k,main}$
    \ENDFOR
    \STATE $\check{V}_{K}(z),\hat{V}_{K}(z) = m_K(z)$,
    \STATE $\mathcal{Z}_K=\{\mathbb{R}^{n_K}\}$
    \FOR{$k\in\{K,\hdots,1\}$, \textbf{for} $l\in\{1,\hdots,N_k\}$}
        \STATE $[\check{V}_{k-1,l}, \hat{V}_{k-1,l}] =$ \texttt{BP}($\{[\check{V}_{k,j}, \hat{V}_{k,j}]\}_{j=1}^{N_{k}}$, $\mathcal{Z}_k$, $\mathcal{Z}_{k-1}^{(l)}$)
    \ENDFOR
    \STATE \textbf{Return:} $\min_{x\in T} \check{V}_{0}(x)$, $\max_{x\in T} \hat{V}_{0}(x)$
\ENDFUNCTION
\end{algorithmic}
\end{algorithm}

% \section{Results}\label{appen:Results}
% \input{sections/appendix/results.tex}

%%%%%%%%%%%%%%%%%%%%%%%%%%%%%%%%%%%%%%%%%%%%%%%%%%%%%%%%%%%%%%%%%%%%%%%%%%%%%%%
%%%%%%%%%%%%%%%%%%%%%%%%%%%%%%%%%%%%%%%%%%%%%%%%%%%%%%%%%%%%%%%%%%%%%%%%%%%%%%%

\end{document}